\documentclass[letterpaper]{article} 
\usepackage[preprint]{aaai2026}  
\usepackage{times}  
\usepackage{helvet}  
\usepackage{courier}  
\usepackage[hyphens]{url}  
\usepackage{graphicx} 
\usepackage{natbib}  
\usepackage{caption} 
\usepackage{algorithm}

\usepackage{newfloat}
\usepackage{listings}
\DeclareCaptionStyle{ruled}{labelfont=normalfont,labelsep=colon,strut=off} 
\floatstyle{ruled}
\newfloat{listing}{tb}{lst}{}
\floatname{listing}{Listing}
\usepackage{comment}
\usepackage{verbatim}
\usepackage{booktabs}
\usepackage{multirow}
\usepackage{amsmath}
\usepackage{amsthm} 
\usepackage{amssymb}
\usepackage{makecell}
\usepackage{xcolor}
\usepackage{mathtools}
\usepackage{longtable}
\usepackage{etoolbox}
\usepackage{algpseudocode}
\usepackage{mleftright,xparse}

    \newcommand\newdot{{\kern.8pt\cdot\kern.8pt}}
\newcommand\nbull{{\kern.8pt\raise1.5pt\hbox{\small\bf .}\kern.8pt}}
\newcommand\1{\hbox{\kern.375em\vrule height1.57ex depth-.1ex
		width.05em\kern-.375em \rm 1}}

\newtheorem{theorem}{Theorem}
\newtheorem{lemma}{Lemma}
\newtheorem{corollary}{Corollary}
\newtheorem{proposition}{Proposition}
\newtheorem{assumption}{Assumption}
\newtheorem{remark}{Remark}

\newcommand\E{\mathbb{E}}

\newcommand\R{\mathbb{R}}
\renewcommand\P{\mathbb{P}}

\newcommand{\lb}{\left(}
\newcommand{\rb}{\right)}

\DeclareMathOperator*{\argmin}{arg\,min}

\DeclareMathOperator{\predmat}{Z}

\DeclareMathOperator{\Fr}{Fr}

 \DeclareMathOperator{\sign}{sign}
 \DeclareMathOperator{\diag}{diag}

\newcommand{\pstar}{{p^{*}}}
\newcommand{\deltastar}{\Delta_{*}}

\newcommand{\pstarexpand}{\max(\max_{i}p_i,\max_j q_j)}

\newcommand{\midmat}{\underline{M}}
\newcommand{\midmatstar}{{\underline{M}^{*}}}

\newcommand{\xboundstar}{\textbf{x}^*}
\newcommand{\yboundstar}{\textbf{y}^*}
\newcommand{\pstarr}{{\mathcal{P}^{*}}}
\newcommand{\bd}{\mathcal{B}}

\newcommand{\observe}{O_M}
\newcommand{\excess}{E_M}

\newcommand{\mathm}{\mathcal{M}}

\newcommand{\midmathatprime}{ {\midmathat}'}

\newcommand{\midmathat}{\widehat{\midmat}}

\newcommand{\midmatprime}{ {\midmat}'}
\newcommand{\midmatprimepand}{R_l\midmat R_r^\top}

\newcommand{\stx}{{X^*}}
\newcommand{\sty}{{Y^*}^\top }
\newcommand{\quadfive}{\quad \quad \quad \quad \quad }

\newcommand{\pstarrexpand}{\max \left( \sqrt{\frac{n}{m}} \xboundstar,\sqrt{\frac{m}{n}} \yboundstar\right)}
\newcommand{\obvs}{\widetilde{G}_{o}}   
\newcommand{\obvsnoo}{\widetilde{G}}   

\DeclareMathOperator{\ground}{G}    
\DeclareMathOperator{\pred}{Z_{\xi^o}}   

\newcommand{\erank}{r}

\newcommand{\andd}{\quad \quad \text{and} \quad \quad }

\newcommand{\rademex}{\widetilde{\mathcal{R}}_N}
\newcommand{\rademexpand}{\frac{1}{N}\sum_{o=1}^N1_{\xi^e_0}S_{\xi^e_o}}
\newcommand{\prerademex}{{\mathcal{R}_N}}
\newcommand{\prerademexpand}{\frac{1}{N}\sum_{o=1}^N1_{\xi^e_0}}
\newcommand{\tosvdize}{\mathcal{H}}
\newcommand{\tosvdizepand}{\frac{1}{M} \sum_{o=1}^M 1_{\xi_{o}}}

\newcommand{\lfnhat}{\widehat{\lfn}}

\newcommand{\nunif}{\Gamma}

\DeclareMathOperator{\lfn}{l}   
\DeclareMathOperator{\lip}{\ell}
\DeclareMathOperator{\entry}{\mathcal{B}}  

\algrenewcommand\algorithmicreturn{\textbf{Return}}

\newcommand{\noisenoo}{\zeta}

\newcommand{\distrimat}{\mathcal{G}}

\newcommand{\GAP}{\text{GAP}}
\newcommand{\disentangle}{\text{Disentangled Estimate}}
\newcommand{\red}[1]{\begingroup\color{red}#1\endgroup}

\title{Generalization Bounds for Semi-supervised Matrix Completion with Distributional Side Information}

\author{
    Antoine Ledent\textsuperscript{\rm 1},
    Mun Chong Soo\textsuperscript{\rm 2},
    Nong Minh Hieu\textsuperscript{\rm 1}
}
\affiliations{
    \textsuperscript{\rm 1} Singapore Management University (SMU),
    \textsuperscript{\rm 2} Binance\\
    {aledent@smu.edu.sg, mc.s@binance.com, mh.nong.2024@phdcs.smu.edu.sg}
}

\begin{document}
\maketitle

\begin{abstract}
  We study a matrix completion problem where both the ground truth $R$  matrix and the unknown sampling distribution $P$ over observed entries are low-rank matrices, and \textit{share a common subspace}. We assume that a large amount  $M$ of \textit{unlabeled} data drawn from the sampling distribution $P$ is available, together with a small amount $N$ of labeled data drawn from the same distribution and noisy estimates of the corresponding ground truth entries. This setting is inspired by recommender systems scenarios where the unlabeled data corresponds to `implicit feedback' (consisting in interactions such as purchase, click, etc. ) and the labeled data corresponds to the `explicit feedback', consisting of interactions where the user has given an explicit rating to the item. Leveraging powerful results from the theory of low-rank subspace recovery, together with classic generalization bounds for matrix completion models, we show error bounds consisting of a sum of two error terms scaling as $\widetilde{O}\left(\sqrt{\frac{nd}{M}}\right)$ and $\widetilde{O}\left(\sqrt{\frac{dr}{N}}\right)$ respectively, where $d$ is the rank of $P$ and $r$ is the rank of $M$. In synthetic experiments, we confirm that the true generalization error naturally splits into independent error terms corresponding to the estimations of $P$ and and the ground truth matrix $\ground$ respectively. In real-life experiments on Douban and MovieLens with most explicit ratings removed, we demonstrate that the method can outperform baselines relying only on the explicit ratings, demonstrating that our assumptions provide a valid toy theoretical setting to study the interaction between explicit and implicit feedbacks in recommender systems.
\end{abstract}

\section{Introduction}

\textit{Matrix completion} (MC) refers to a broad class of statistical problems where one wishes to recover the entries of an unknown ground truth matrix $\ground\in\R^{m\times n}$ based on a set of $N\ll mn$ potentially noisy observations. In short, it is a supervised learning problem where the independent variable is a (row, column) pair where both components can only take a finite set of values $[m]$ or $[n]$.  Despite its apparent simplicity, this problem is not only of high practical relevance (in recommender systems~\citep{koren09,IGMC}, chemical and thermal engineering~\citep{selfcite,hansch2025data} and drug discovery~\citep{IMCforDrug}, etc.), but also surprisingly challenging and nuanced in its statistical properties. The earliest and most well-known works in the field focused on the \textit{exact recovery} problem: the celebrated works of~\citet{Genius,CandesRecht} showed that minimizing the nuclear norm of the candidate matrix $Z$ subject to $Z_{i,j}=G_{i,j}$ whenever the entry $(i,j)$ is in the set $\Omega$ of observed entries provably recovers the exact ground truth matrix as long as the number of observations taken uniformly at random is larger than $\widetilde{O}\left(nr\right)$ where $n$ is the size of the matrix and $r$ is the ground truth rank.  However, many different regimes can be considered. For instance, some recent work refine the bounds of~\citet{Genius,noisycandes} to incorporate a very fine joint dependence on the subgaussianity constant of the noise and the size and rank of the matrix~\cite{ChenChi20}, whilst other study excess risk bounds in nonuniform sampling regimes~\citep{foygel2011learning,ReallyUniform1,ReallyUniform2} or out-of-distribution generalization for certain missingess patterns~\citep{ma2019MNAR}.

A notable set of works consider the so-called `Inductive Matrix Completion' (IMC) setting~\cite{IMCtheory1,mostrelated,LedentIMC}, where one assumes that the learner has access to side information matrices $X\in\R^{m\times d}$ and $Y\in\R^{n\times d}$ with the property that the ground truth matrix $\ground$ can be represented as $\ground=XM Y^{\top}$  for some unknown matrix $M$. The role of the matrices $X,Y$ is to represent some external knowledge about each value of the underlying discrete variable. For instance, in recommender systems, the rows of $X$ and $Y$ would correspond to feature vectors describing the users and items respectively. Similarly, in a drug interaction prediction context, the rows of $X,Y$ would consist of feature vectors containing information about the chemical composition of each drug. Thus, the IMC framework aims to take matrix completion closer to real applications by incorporating available side information, and continues to attract interest in recent years, with many different variants being proposed. For instance, ~\citet{jalan2025optimal} study a challenging and biologically-inspired setting where entire rows or columns are missing and study both passive and active sampling regimes, where the user can select whole rows and columns to observe. However, a weakness of existing IMC approaches is the assumption that the side information matrices $X,Y$ are known \textit{a priori}, and the column (resp. row) space of the ground truth matrix is exactly contained in the column space of $X$ (resp. $Y$). In practice, such side information needs to be \textit{estimated} from other forms of training data (social media user graph, molecular drug structure, etc.).

To the best of our knowledge, all theoretical results on matrix completion to date assume that the available samples are all labeled: every observed entry $(i,j)\in[m]\times [n]$ is accompanied by a (possibly noisy) estimate of the ground truth entry $\ground_{i,j}$. However, in many real-life scenarios, it is common for a much larger amount of unlabeled samples to be naturally available. In a recommender system, one may observe a large set of implicit interactions $\Omega$, where the presence of the pair $(i,j)$ in $\Omega$ indicates that user $i$ has watched/consumed the item $j$. This type of interaction, often referred to as `implicit feedback' is distinct from `explicit' interactions where a \textit{rating} (typically on a scale between 1 and 5 stars) is given by user $i$ to item $j$: such explicit interactions might be much more scarce. Thus, there is a need for a more inclusive learning setting in matrix completion to incorporate semi-supervised settings. Our contributions are as follows: 
\begin{itemize}
    \item We propose a semi-supervised learning paradigm for matrix completion: we assume that the sampling distribution $P\in[0,1]^{m\times n}$ over entries shares a low-rank subspace with the ground truth matrix $\ground$, and that that a large number $M$ of \textit{unlabeled} interactions is provided, together with a much smaller number $N$ of \textit{labeled} interactions.
    \item Under incoherence, \textit{uniform marginals} and bounded sampling probability assumptions, we leverage powerful results from matrix pertrubation theory to prove generalization bounds which scale like $\widetilde{O}\left(\sqrt{\frac{nd}{M}}\right)+\widetilde{O}\left(\sqrt{\frac{dr}{N}}\right)$, showing that the estimation errors associated to the subspace recovery (from unlabeled samples) and matrix recovery (from labeled samples) can be \textit{disentangled}. 
    \item In synthetic data experiments, we demonstrate that the generalization error can indeed be closely approximated by a sum of terms corresponding to each form of error. 
    \item In RecSys datasets, we demonstrate that a large number of unlabeled samples can substantially improve the performance of explicit feedback prediction methods. This aligns with the conclusions implicitly in recent works~\cite{IGMC,xia2022hypergraph,ledent2025conv4rec}, lending legitimacy to our learning paradigm. 
\end{itemize} 

\section{Related Works}

\textbf{Matrix completion in the i.i.d. setting (without side information):} Our results are most closely related to the so-called `approximate recovery' branch of the literature on matrix completion, which assumes that the observations are drawn i.i.d. from a distribution over entries, and that the performance measure is the (in-distribution) excess risk. In this setting, a sample complexity of $\widetilde{O}\left(nr\right)$ was achieved for empirical risk minimization with a nuclear norm constraint in~\citet{foygel2011learning} under a uniform marginals assumption, whilst a sample complexity of $\widetilde{O}\left(n^{\frac{3}{2}} \sqrt{r}\right)$ was achieved under arbitrary sampling regimes in~\citet{ReallyUniform1,ReallyUniform2}. Before that, similar settings were studied in~\citet{early}, which includes the case of explicit rank restriction in a classification setting. When imposing Schatten $p$ quasi norm constraints, recent results show a sample complexity of $\widetilde{O}\left( n^{1+\frac{p}{2}} r^{1-\frac{p}{2}}\right)$ and $\widetilde{O}\left(nr\right)$ in the arbitrary sampling and uniform marginals settings respectively~\cite{ledent2024generalization}. Very tight bounds in terms of the dependency on the variance of the noise were proved in~\citet{spectralnormmagic} and~\citet{spectralmagic} for nuclear norm and exactly low-rank cases respectively. However, all of those results depend at least linearly in the size $n$ of the matrix and require all of the samples to be labeled, making the results ineffective in the semi-supervised learning setting we aim to study. 

Another closely related branch of the literature studies `\textbf{Inductive Matrix Completion}' (with side information), which also studies the in-distribution i.i.d. setting but assumes the presence of known side information matrices $X\in\R^{m\times d}$ and $Y\in\R^{n\times d}$ such that the ground truth matrix is known to be representable in the form $\ground=XM Y^\top$ for some $M\in\R^{d\times d}$. Similarly to our work, much of the literature on this problem imposes nuclear norm constraints on the core matrix $M$. The problem was initially proposed in~\citet{IMCtheory1}, which studied the exact recovery setting under the uniform sampling regime with nuclear norm minimization, and was later studied under the i.i.d. setting for the first time in~\citet{mostrelated,mostrelatedearly}, where sample complexity bounds of $\widetilde{O}(d^2r)$ are provided. The bounds were improved to $\widetilde{O}\left(d^{\frac{3}{2}}\sqrt{r}\right)$ and $\widetilde{O}\left(dr\right)$  in the arbitrary sampling and uniform marginals cases respecitvely in~\citet{LedentIMC}, and bounds with a finer dependence on the variance of the noise were proved in~\citet{ledent2023generalization}. Whilst those rates are similar to our bounds on the estimation error arising from the labeled data, they do not involve subspace estimation: in fact, we rely on those results as tools to establish our own bounds, but the proofs are completely different. The main difficulty in our work is bounding the estimation of the subspaces $X,Y$ using the unlabeled data and controling the propagation of that error through the downstream inductive matrix completion problem. In particular, our results require more stringent assumptions on the sampling distribution as a result of the increased technical difficulty. 

Many of the tools we rely on for the subspace estimation problem come from the matrix \textbf{perturbation theory }and its applications to datascience~\citep{spectralmagic}. Indeed, this book also provides a remarkable array of strong bounds for matrix completion. However, the bounds apply to a different regime: the entries are sampled uniformly at random with Bernouilli sampling (without replacement and without side information) and the emphasis of the results lies in a tighter dependence on the subgaussianity constant of the noise, as well as providing entry-wise estimates. Whilst the book~\citet{spectralmagic} deals mostly with models involving explicit rank restriction, we note that similar results are also known for nuclear norm minimization~\cite{ChenChi20}.

Another nonuniform sampling regime which has gained a lot of attention in recent years is the so called `Missing Not at Random' (\textbf{MNAR}) matrix completion problem~\citep{ma2019MNAR,choi2024matrix,jalan2025optimal},  where the entries are sampled with independent Bernouilli masks whose associated probabilities are obtained by applying a sigmoid function to an unknown low-rank matrix. Thus, this corresponds to an alternative form of nonuniform sampling compared to our i.i.d. setting, and this research direction also involves a low-rank constraint in the sampling distribution. However, there are many notable differences: first, the emphasis is on bounding the (uniform) Frobenius error, making this an \textit{out-of-distribution} problem where the empirical risk is reweighted by inverse propensity scores to compensate for the nonuniformity of the sampling distribution. Second, the `low-rank' condition on the sampling distribution is $\mathcal{P}=\sigma(\Gamma)$ where $\mathcal{P}$ is the matrix of Bernouillli probabilities (the `propensity scores') and a nuclear norm constraint is imposed on $\Gamma$. Whilst this is somewhat comparable to a low-rank condition, it is important to note that the choice $\Gamma\simeq 0\in\R^{m\times n}$ (which is both low-rank and low nuclear norm) leads to a uniform masking probability of $0.5$, which is a \textit{dense} observation regime at odds with the sparse observations regime studied in classic matrix completion settings. In addition to the difference in performance measure, this further complicates any comparison betweeen the results in~\citet{ma2019MNAR} and the approximate recovery literature at the common regimes where very few entries are observed. Third, to compensate for the nonuniformity of the sampling distribution despite the uniform performance measure, an inverse multiplicative factor in the minimum sampling probability for any entry is present in~\citet{ma2019MNAR}, which further distinguishes its results from our own, which apply even if many entries have zero sampling probability. Whilst the problem setting in ~\citet{jalan2025optimal,mcgrath2024learnertransferlearningmethod} also involves matrices with shared subspaces, they are assumed to come from distinct sources and both matrices are partially observed, instead of corresponding to low-rank sampling distributions. Similarly, the idea of a low-rank distribution is studied in~\citet{anandkumar14,vandermeulen2021beyond,Sidibig1} and used in a recommender systems setting in~\citet{poernomo2025probabilistic}, but none of these works involve explicit feedback or a shared subspace.  Lastly, there is a rich literature on semi-supervised learning in a broader machine learning context~\citep{balcan2005pac,blum1998combining,bekker2020learning}. However, the key techniques such as Contrastive Learning~\cite{CRL3,hieuAAAI,hieu2025generalization,alves2024context,ghanooni2024generalization,ghanooni2025mitigating}, typically apply to classification problems rather than regression setting, and such methods usually do not apply to discrete input spaces, such as matrix completion, which comes with its own challenges. 


\section{Main Results}

\subsection{Learning Setting and Assumptions}

\label{sec:assum}

\textbf{Matrix Completion in the i.i.d. setting. } 
We consider a noisy matrix completion problem in the i.i.d. (regression) setting. This setting has been studied in the following papers, however, we reintroduce it with our notation (and in slightly greater formality and generality) for the sake of convenience~\cite{foygel2011learning,ReallyUniform1,ReallyUniform2,LedentIMC,mostrelated,mostrelatedearly}, as it differs from both the problem of exact matrix recovery~\cite{Genius,noisycandes,spectralmagic} or `missing not at random' matrix completion.  The sampling procedure for each sample is i.i.d. according to the following distribution. A full labeled sample $(\xi,\obvsnoo)$ consists in: 
\begin{itemize}
	\item  An \textit{entry} $(i,j)=\xi\in[m]\times [n]$ sampled  from a categorical probability distribution with Probability Mass Function (PMF) given by $P\in\R^{m\times n}$  over $[m]\times [n]$ (i.e. $\sum_{i\leq m,j\leq n}P_{i,j}=1$), and 
	\item A \textit{label} $\obvsnoo\in\R$ drawn from a conditional distribution $\distrimat_{i,j}$ which depends on the entry $\xi=(i,j)$.
\end{itemize}

The entry is considered as the independent variable and the label is considered as the target variable in the sense of supervised learning. Predictors are functions $\predmat\in\R^{m\times n}$ and their performance on a test sample $(\xi,\obvsnoo,\hat{y})$ is measured via a loss function $\lfn: ([m]\times [n]) \times \R \times \R \rightarrow \R^+: (\xi,\obvsnoo,\hat{y})\mapsto \lfn(\xi,\obvsnoo,\hat{y})$. The population level performance of the predictor $\predmat$ is the population expected risk $\lfn(\predmat)=\E_{\xi,\obvsnoo}\left( \lfn(\xi,\obvsnoo, \predmat_{\xi})    \right) $. Often, we also consider its empirical analogue $ \widehat{\E}_\xi\left( \lfn(\xi,\obvsnoo, \pred_{\xi})    \right)=\frac{1}{N} \sum_{o=1}^N \lfn(\xi_o, \obvs,\pred)$, where $\xi_1,\ldots,\xi_N\in [m]\times [n]$ denotes a set of empirical samples. The Bayes predictor or \textit{ground truth matrix} $\ground\in\R^{m\times n}$ is defined by $\ground_{i,j}\in\argmin_{\hat{y}\in\R} \E_{\xi,\obvsnoo}\lfn(\xi,\obvsnoo,\hat{y})$. For instance, if $\lfn$ is the square loss (and in particular, doesn't depend on the entry $(i,j)$) and the observations $\obvsnoo$ are equal to $R_{\xi}+\noisenoo$ for some fixed matrix $R\in\R^{m\times n}$ and some noise $\noisenoo$ satisfying $\E(\noisenoo)=0$, then $R=\ground$.

Going further, we consider a \textbf{semi-supervised setting} where the learner has access to 
\begin{itemize}
	\item $N$ \textit{labeled} observations  $ \{(\xi_1^e, \tilde{G}_1), $ $\dots, (\xi_N^e, \tilde{G}_N)\}$ $:=\mathcal{S}_N $,  drawn i.i.d. from the distribution above, and
	\item $M$ i.i.d. \textit{unlabeled} samples $(i_o,j_o)$ (for $o\leq M$) drawn from the distribution $P$  alone, where we do not have access to the label $\obvsnoo$. 
\end{itemize}   

We write $\observe=\frac{1}{M} \sum_{o=1}^M 1_{i_o,j_o}$ for the matrix of observed \textit{unlabeled} samples (i.e., the empirical analogue of the PMF $P$), and make the following key assumptions. 

\begin{assumption}[Boundedness and Lipschitzness of the Loss Function]
	\label{assum:loss}
	The loss function $\lfn$ is uniformly bounded by $\entry$ and for any value of $(i,j)$, is Lipschitz continuous with Lipschitz constant $\lip$: 
	For all $(i,j)\in[m]\times [n]$, 
	\begin{align}
		\left|		\lfn((i,j),y,\hat{y}_1)\right| &\leq \entry  \quad \quad \text{and}\\ 
		\left|		\lfn((i,j),y,\hat{y}_1)-		\lfn((i,j),y,\hat{y}_2)\right|&\leq \lip  \left|\hat{y}_1-\hat{y}_2\right|.
	\end{align}
\end{assumption}

\begin{assumption}[Existence of Shared Low-rank Subspace]
	\label{assum:lowrankshared}
	We assume that the sampling distribution $P$ has low-rank $d$, i.e. $P=U^{*} \Sigma^{*} [V^{*}]^{\top}$  where $U^*\in \R^{m\times d}$ and $V^*\in  \R^{n\times d}$. We assume that the ground truth matrix can be represented as $\R^{m\times n} \ni \ground = U^* \midmatstar^{-} [V^{*}]^{\top}$ for some matrix $\midmatstar^{-}$. For convenience and to stick to the scaling used in other literature on inductive matrix completion, we introduce the notation $X^*:= \sqrt{\frac{m}{d}}U^*$,  $Y^*:= \sqrt{\frac{n}{d}}V^*$ and $\midmatstar:= \sqrt{\frac{d^2}{mn}}\midmatstar^{-} $.  Thus we certainly have 
	\begin{align}
		\label{eq:realisability}
		\ground = X^{*} \midmatstar  [Y^*]^{\top}.
	\end{align}
	
	We also assume that $\|\midmatstar\|_*\leq \mathm$ for some constant $\mathm$. 
\end{assumption}
This assumption forms the basis of our novel learning setting: to the best of our knowledge, it is the first attempt to formalize the existence of a relationship between implicit feedback (the sampling distribution over ratings) and explicit feedback (the matrix of latent rankings). We choose this assumption because this is the simplest way to assume such a relationship in a matrix completion setting.

\begin{assumption}[Incoherence of the Shared Low-rank Subspaces]
	
	We also write $\xboundstar = \max_{i\leq m} \|[X^*]_{i,\nbull}\|$ and $\yboundstar = \max_{j\leq n} \|[Y^*]_{j,\nbull}\|$, which can be interpreted as incoherence measures for the row and column subspaces of the ground truth $\ground$ and are treated as $O(1)$ constants. We also define the following relevant quantity, which can be interpreted as an overall incoherence constant:
	\begin{align}
		\pstarr=\pstarrexpand.
	\end{align} 	
\end{assumption}

\begin{assumption}[Approximately Uniform Marginals]
	\label{assum:unimarg}
	We assume that the marginals of $P$ are bounded uniformly as follows, for some constant $\kappa_1$ 
	\begin{align}
		\label{eq:updatedcondsrepeat}
		p_i\leq \frac{\kappa_1}{m}  \andd q_j \leq \frac{\kappa_1}{n},
	\end{align}
	where $p_i:=\sum_{j\leq n}P_{i,j}$ and $q_j:=\sum_{i\leq n}P_{i,j}$ are the marginals for the $i$th row and $j$th column respectively. 
\end{assumption} 

This assumption is common in works on approximate recovery in matrix completion. For instance, ~\citet{foygel2011learning,ReallyUniform1} and~\citet{ledent2024generalization} use it for some of their stronger results (typically, a non-trivial result is shown without this assumption, and a tighter bound is shown with this assumption). This assumption is weaker than the uniform sampling assumption which is typical in exact recovery results such as those of~\citet{CandesRecht,SimplerMC,spectralmagic,ChenChi20}.

\begin{assumption}[Well-conditioning of PMF]
	\label{assum:wellconditioned}
We assume the sampling distribution $P$ is well conditioned. More precisely, we rely on the following conditioning number in our bounds: 
\begin{align}
	\label{eq:defineconditionnumber}
	\kappa_*= \frac{\|P\|}{\deltastar},
\end{align}
where $\deltastar$ is the last singular value of $P$.
\end{assumption}

\begin{assumption}[Well-conditioning of $X,Y$]
	\label{assum:kappa2}
	
	We as assume the following bound on the spectral norm of the ground truth side information matrices $X,Y$:
		\begin{align}
		\label{eq:kappa2cond_main}
		\|X^*\|\leq \xboundstar \sqrt{\kappa_2\frac{m}{d}} \andd 	 	\|Y^*\|\leq \xboundstar \sqrt{\kappa_2\frac{n}{d}}.
	\end{align}
\end{assumption}

\begin{assumption}[Bound on Maximum Sampling Probability]
	\label{Assum:gamma}
	We assume that the maximum possible entry of the sampling distribution $P$ is bounded by a constant $\nunif$. 
	Thus,  $\nunif$ is defined as $\max_{i,j}P_{i,j}mn$ so that for all $i,j$: 
	\begin{align}
		P_{i,j}\leq \frac{\nunif}{mn}.
	\end{align}
\end{assumption}

Assumption~\ref{Assum:gamma}, which requires a uniform upper bound on the sampling probability for any entry, is the most restrictive of ours. Still, Assumption~\ref{assum:unimarg} implies that Assumption~\ref{Assum:gamma} is always satisfied with at least a coarse estimate $\nunif \leq \kappa_1 [m+n]$.  Relying on this still yields non-trivial results, but at the cost of an assumption of the form $N\geq \frac{m+n}{2}$: one needs at least $O(1)$ \textit{labeled} interactions in each row/column. Whilst this is a significant restriction (because the sample complexity in terms of labeled examples isn't truly independent of the size of the side information $d$), the result is still of interest as this is an absolute threshold rather than a true contribution to the error bound. If $\nunif$ is constant then the results hold without this caveat.  Lastly, we note that, in contrast with  the \textit{lower} bound on the sampling probability in~\citet{ma2019MNAR}, even Assumption~\ref{Assum:gamma} with an \textit{absolute} constant $\Gamma$ covers non-trivial cases and doesn't necessarily imply that the sampling distribution is approximately uniform. 
 Indeed, suppose that the $n$ rows and columns are each divided into $k$ `groups' or clusters. One can visualize this as a `check board' where the 1st $n/k$ rows and cols belong to the 1st group, but the example is more general: group memberships can be unknown. Assumption~\ref{assum:unimarg} implies that both the ratings $G_{i,j}=\tilde{G}_{g(i),g(j)}$ and the probabilities $P_{i,j}=\tilde{P}_{g(i),g(j)}$ only depend on the (unknown) user and item groups. Assumption~\ref{Assum:gamma} only concerns $\tilde{P}\in\mathbb{R}^{k\times k}$ and is independent of $n$. For instance, if $\tilde{P}=\frac{1}{2} I/k+\frac{1}{2}U$ where $U$ is uniform, then $\Gamma=k/2+1/2=[d+1]/2$ and the term $O(\sqrt{\frac{\Gamma nr}{MN}})$ $\leq O(\sqrt{\frac{d}{N}}\sqrt{\frac{nr}{M}})$ in our results below is benign. In this case, the assumptions hold with $\Gamma=O(d)$, $\kappa^*,\kappa_1,\kappa_2,\in O(1)$. In fact, user/item clustered settings are related to the stochastic block model~\cite{sbm1,sbm2} and have been studied in various works both within~\cite{omic,autobomic,alves2024regionalization} and outside~\cite{Vincent} matrix completion, achieving strong performance in both cases.

\subsection{Model}

We assume that the learner is aware of the existence of a shared low-rank subspace and proceeds as follows: 
\begin{itemize}
	\item\textbf{Step 1:} first, the unlabeled data is used to estimate the subspaces via singular value decomposition: the matrix $\tosvdize$ is constructed, and a singular value decomposition of order $d$ is performed on it, yielding the SVD $\tosvdize= U\Sigma V^\top$. Then, the side information matrices $X,Y$ are constructed as $X = \sqrt{\frac{m}{d}} U$ and $Y = \sqrt{\frac{n}{d}} V$.
	\item\textbf{Step 2:} next,  after fixing an upper bound constraint $\mathm$ for the nuclear norm of $\midmat$, and the ground truth matrix is estimated via empirical risk minimization using the classic Inductive Matrix Completion (IMC) algorithm: 	
    \begin{align}
        \label{eq:inuniformalgotheorymain}
	\midmat &= \argmin_{\|\midmat\|_* \leq \mathm} \lfn(X\midmat Y^\top) \\
        &= \argmin_{\|\midmat\|_* \leq \mathm} \frac{1}{N} \sum_{o=1}^{N} \lfn\left( (X\midmat Y^\top)_{\xi_{o}}, \obvs \right). \nonumber
    \end{align}
\end{itemize}

We will also use the notation $\erank$ for the quantity $\frac{\mathm^2}{d^2}$, which scales as the rank of the matrix $\midmat$ (and therefore, $\ground$): although this is a real number which depends on a tunable parameter $\mathm$, in the case of a homogeneous spectrum and $O(1)$ entries, setting $\mathm$ large enough to ensure that $\erank$ is $O(\text{rank}(\ground))$ will guarantee that the ground truth is representable. In the noiseless case, this guarantees that all our bounds also hold for the Population Risk. See~\citet{foygel2011learning} and~\citet{LedentIMC} for more details.

\textbf{Remark:} In applications, the minimization problem from eq.~\eqref{eq:inuniformalgotheory} is replaced by a Lagrangian form, which can be solved with gradient methods in Pytorch.  The equivalence between the regularizer term $\left[\|A\|_{\Fr}^2  +\|B\|_{\Fr}^2\right] $  and the nuclear norm of $\midmat$ is a consequence of the classic Lemma 6 in~\cite{softimpute}. The precise algorithm can be found in Alg.~\ref{alg:DAMC}. 
\begin{align}
\textbf{minimize} \> \frac{1}{N} \sum_{o=1}^{N} \lfn\left( (XAB^\top Y^\top)_{\xi_{o}}, \obvs \right)\\+\lambda_{\mathm} \left[\|A\|_{\Fr}^2  +\|B\|_{\Fr}^2\right] \label{eq:inuniformalpractical}
\end{align}

\begin{algorithm}
	\caption{DAMC (Distributionally Aware Matrix Completion)}\label{alg:DAMC}
	\begin{algorithmic}[1]
		\Require Observed unlabeled data $\mathcal{S}_M = \{\xi_1, \dots, \xi_M\}$,
		Observed labeled data $\mathcal{S}_N = \{(\xi_1^e, \tilde{G}_1), \dots, (\xi_N^e, \tilde{G}_N)\}$, parameters $d$ (size of the side information), and upper bound constraint $\mathm $ on the nuclear norm of the core matrix.
		\Ensure Predictions $X\midmat Y\in \mathbb{R}^{m\times n}$
		\State Construct the matrix $\tosvdize= \frac{1}{M} \sum_{o=1}^{M} \xi_{o}$.
		\State Compute the truncated SVD $\tosvdize= U\Sigma V^\top$ (up to rank $d$) and define $X = \sqrt{\frac{m}{d}} U$ and $Y = \sqrt{\frac{n}{d}} V$.
		\State Solve the optimization problem:
		\begin{align}
			\midmat &= \argmin_{\|\midmat\|_* \leq \mathm} \lfn(X\midmat Y^\top) \nonumber \\
                    &= \argmin_{\|\midmat\|_* \leq \mathm} \frac{1}{N} \sum_{o=1}^{N} \lfn\left( (X\midmat Y^\top)_{\xi_{o}}, \obvs \right) \label{eq:inuniformalgotheory}
		\end{align}
		\State \Return Core matrix $\midmat$; side information matrices $X$, $Y$, and matrix of predictions $\widehat{Z} = X\midmat Y^\top$.
	\end{algorithmic}
\end{algorithm}

\subsection{Main Results}


\begin{theorem}
	\label{thm:unif_generalization_bound_new_main}
	
Instate Assumptions~\ref{assum:loss},~\ref{assum:lowrankshared},~\ref{assum:unimarg}, ~\ref{assum:wellconditioned},~\ref{assum:kappa2},and~\ref{Assum:gamma}, then: 
	
	\begin{align}
		\label{cond:Magain??}
		M\geq 470  \log\left(\frac{4[m+n]}{\delta}\right)   \kappa_*^2\pstarr^2[m+n].
	\end{align}
	
	With probability greater than $1-\delta$ over the draw of both the implicit and explicit feedbacks,  the following generalization bound  holds simultaneously over any predictor  $X\midmat Y^\top\in\R^{m\times n}$  for $\midmat\in\R^{d\times d}$ such that $\|\midmat\|\leq \mathm$
	\begin{align}
	&\lfn \left(X\midmat Y^\top \right) -\lfnhat\left( X \midmat Y^\top  \right) \leq \\
        &\frac{2\entry \log(6/\delta )}{\sqrt{N}} + 16\lip \pstarr^2 \sqrt{\kappa_1\kappa_2} \log(2de) \sqrt{\frac{d\erank }{N}}\nonumber + \\ 
        &75\pstarr \ell  \kappa_*\kappa_1\log\left(\frac{12[m+n]}{\delta }\right)\sqrt{\frac{[m+n]\erank}{M}} +  \nonumber \\
        &25\pstarr \ell  \kappa_*\log\left(\frac{12[m+n]}{\delta }\right)\sqrt{\frac{[m+n]\nunif \erank}{MN}} \nonumber 
	\end{align}
	where as usual, $\lfn(Z):=\E\left( \lfn(Z_{\xi},\obvsnoo)\right)$  and $\lfnhat(Z):=\frac{1}{N} \sum_{o=1}^N \lfn(Z_{\xi^e_o},\obvs)$.
	
\end{theorem}

We also have the following immediate consequence in terms of excess risk. 
\begin{corollary}
	\label{cor:unimargfinal_new_main}
	
	
	Let $\widehat{Z}$ be the matrix output by algorithm~\ref{alg:DAMC}, under assumptions 1-7, we have the following excess risk bound w.p. $\geq 1-\delta$ : $\E(\lfn(\widehat{Z}, \obvsnoo))-\widehat{\E}(\lfn(\widehat{Z}, \obvsnoo)) \leq$
    \begin{align}  
        &O\left[  [ \lip+\entry]  \kappa_* \pstarr\kappa_1\log\left(\frac{[m+n]}{\delta }\right)\sqrt{\frac{[m+n]\erank}{M}}\right] + \nonumber \\
        &O\left[  [ \lip+\entry]  \kappa_* \pstarr\kappa_1\log\left(\frac{[m+n]}{\delta }\right)\sqrt{\frac{\Gamma[m+n]\erank}{MN}}\right] + \nonumber \\
	&O\left[\lip \pstarr^2 \sqrt{\kappa_1\kappa_2} \log(2de) \sqrt{\frac{d\erank }{N}}  \right] \nonumber.
    \end{align}
\end{corollary}

If only Assumptions 1-6 hold, then one can replace $\Gamma$ (from Assumption 7) by the cruder estimate $\Gamma= \kappa_1 [m+n]$ which can be obtained from assumption~\ref{assum:unimarg} instead. This allows us to replace Assumption~\ref{Assum:gamma} by 
    \begin{align}
        N\geq \frac{m+n}{2}, \quad \text{so that we have}\label{cond:itchy}
    \end{align}
    \begin{align}
	&\E(\lfn(\widehat{Z}, \obvsnoo))- \widehat{\E}(\lfn(\widehat{Z}, \obvsnoo)) \leq \\
        &\widetilde{O}\Bigg[[\lip+\entry]  \kappa_* \pstarr\kappa_1^{3/2}\sqrt{\frac{[m+n]\erank}{M}} +  \lip \pstarr^2 \sqrt{\kappa_1\kappa_2} \sqrt{\frac{d\erank }{N}}\Bigg],\nonumber
	\end{align}
	where the notation $\widetilde{O}$ hides polylogarithmic factors in $m,n,\delta$. Although the condition~\eqref{cond:itchy} implies the number of labeled samples $N$ must pass a threshold which is linear in the size of the full matrix, this is a fixed threshold which doesn't depend on the desired error level. In addition, this is only a worst-case scenario: one can alternatively assume that Assumption 7 holds with $\Gamma\simeq d$, also allowing the absorption of the higher order term $\sqrt{\frac{\Gamma[m+n]\erank}{MN}}$ into the others. 

In all cases, treating $\kappa_1,\kappa_2,\kappa_*, \pstar, \lip, \entry$ as constants and under either assumption 7 with $\Gamma= O(d)$ constant or condition~\eqref{cond:itchy}, we see that the error scales as $$\widetilde{O}\left(  \sqrt{\frac{[m+n]\erank}{M}} + \sqrt{\frac{d\erank }{N}} \right).$$ 

Here, the first term corresponds to the error in the estimation of the shared low-rank subspaces with the unlabeled data, and the second one corresponds to the error in estimating the ground truth matrix assuming a perfect knowledge of the side information matrices $X,Y$. Thus, accurate recovery can be performed as long as we have at least $M=\widetilde{O}\left( [m+n]\erank\right)$ unlabeled samples and $N=\widetilde{O}\left(d\erank\right)$ \textit{labeled} samples. Thus, the result shows that the errors corresponding to the estimation of the common subspace and the estimation of the ground truth matrix based on the subspace information only combine additively. In particular, this implies that successful recovery of the ground truth matrix (in terms of in-distribution excess risk) is possible with only a very small number of labeled samples, as long as a larger number of unlabeled samples is available. This conclusion is of interest in the field of recommender systems, where the labeled interactions correspond to `explicit feedback' and the unlabeled interactions correspond to `implicit feedback'. 

This observation is in sharp contrast to the sample complexities which can be obtained via a direct application of existing MC results (ignoring the unlabeled samples): if we were to apply the state-of-the-art results for matrix completion in the i.i.d. setting with uniform marginals without the use of the side information matrices $X,Y$, one would obtain a bound of $\widetilde{O}\left(\sqrt{\frac{[m+n]\erank}{N}}\right)$ (cf.~\citet{foygel2011learning}): the number of labeled samples would need to be as high as $\widetilde{O}\left([m+n]\erank\right)$. In contrast, we only require  $\widetilde{O}\left(d\erank\right)$ labeled interactions. A remarkable property of this result is that the bound is meaningful even if the average number of labeled samples per row/column is vanishingly small, as long as there are $\widetilde{O}\left(\erank\right)$ \textit{unlabeled} samples in each row or column.

\textbf{Remarks on the value of $d$ in Assumption~\ref{assum:lowrankshared} and Assumptions~\ref{assum:wellconditioned} and~\ref{assum:kappa2}:} Well-conditioning assumptions such as Assumption~\ref{assum:lowrankshared} and Assumptions~\ref{assum:wellconditioned} and~\ref{assum:kappa2} are common in matrix perturbation theory~\cite{spectralmagic}. We note that true rank $r$ of the matrix $\ground$ can be much smaller than the dimension $d$ of the shared low-rank subspace. Indeed, the core matrix $\midmatstar$ can be low-rank. Thus,  Assumption~\ref{assum:lowrankshared}, ~\ref{assum:wellconditioned} and~\ref{assum:kappa2} can be summarized as follows: (1) the sampling distribution $P$ is well approximated by a rank $d$ matrix whose row and column spaces include those of the ground truth and (2) the row and column spaces of the ground truth $\ground$ are (possibly strict) subspaces of those of the sampling distribution $P$. Cf. Remark 1 in the Appendix for more details.

\section{Experiments}

We perform both synthetic and real data experiments to validate the pertinance of our bounds. Our key claims are: 

\begin{itemize}
    \item \textbf{Claim 1} The errors stemming from the subspace estimation (with the unlabeled samples) and (inductive) matrix completion  components only combine additively.
    
    \item \textbf{Claim 2} In recommender systems datasets, unlabeled interaction data (often referred to as `implicit feedback') contains relevant information to estimate the row and column subspaces of the ground truth matrix of labeled interaction data (containing the ratings from 1 to 5 given by each user to the interacted items).
\end{itemize}

We validate Claim 1 with synthetic data experiments and Claim 2 with real data respectively.

\subsection{\textcolor{black}{Synthetic Data Experiments}}

To demonstrate the decomposition of the error into independent terms corresponding to the estimation errors of the labeled and unlabeled data respectively, we generated $\ground,P\in\R^{200\times 200}$ satisfying  our assumptions with $d=\erank=4$. The detailed experimental setup can be found in the Appendix. For a broad range of values of both $N$ and $M$ ($M\in\{ 10000, 20000, \ldots, 100000\}$ and $N\in\{50, 100, 150, \ldots, 1000\}$), we evaluate the average generalization error over $30$ independent runs. This range is selected because $N=100, M=100000$ results in perfect recovery up to a high decimal point. We compared two quantities: 
\begin{itemize}
	\item[1] The generalization gap (test error $-$ training error), and 
	\item[2] A disentangled estimate of the generalization error calculated as follows: 
	\begin{align}
		\label{eq:disentangle}
		&\disentangle(M,N)=\nonumber \\
            & \GAP(M,1000)+ \GAP(100000,N), 
	\end{align}
	where $1000$ an $100000$ are the maximum possible values for $N$ and $M$ respectively. 
\end{itemize}

Thus, the two terms in equation~\eqref{eq:disentangle} can be interpreted as corresponding to the errors introduced from the estimation of the subspace and the ground truth matrix respectively. The results are presented below in Figure~\ref{fig:synthetic}.

\begin{figure}[htbp]
	\centering
	\includegraphics[width=1.0\linewidth]{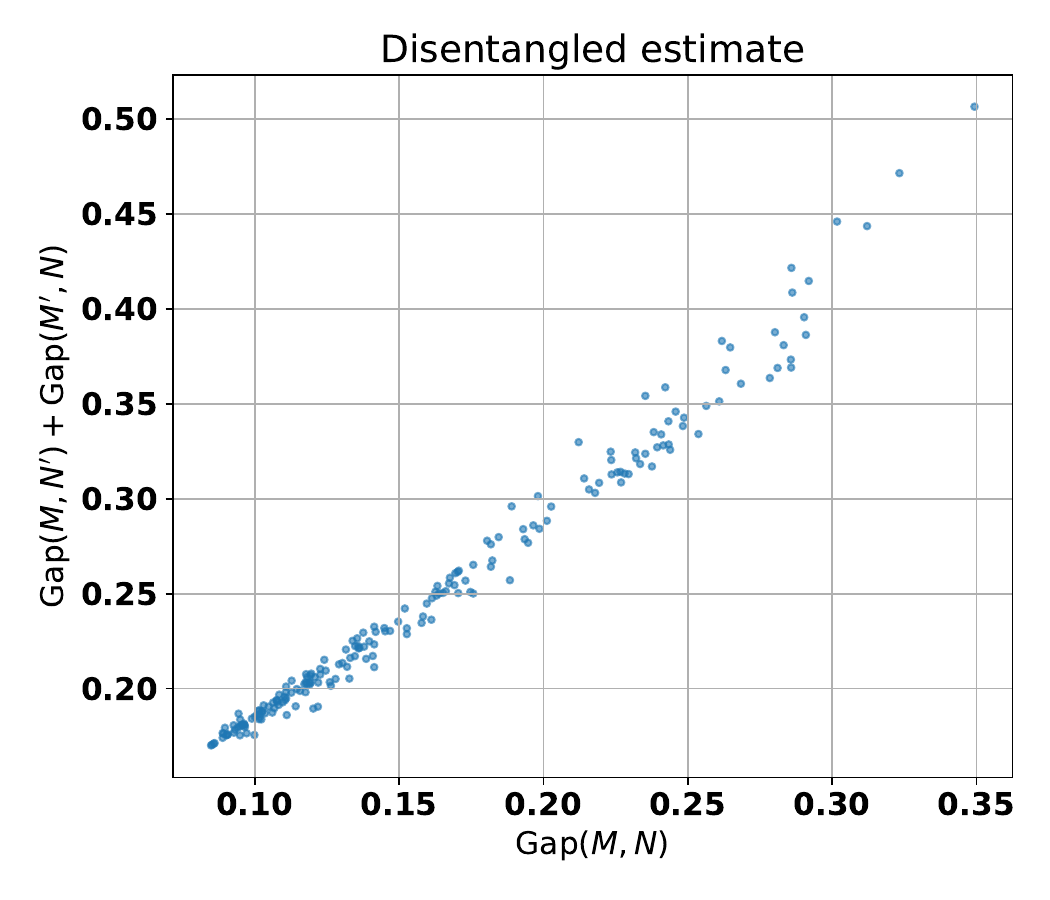}
	\caption{Comparison of generalization error (x-axis) and the corresponding disentangled estimate (y axis) in the synthetic dataset. Each point in the scatter plot corresponds to one configuration $(M,N)$, with the results averaged over 30 independent runs. }\label{fig:synthetic}
\end{figure}
We observe a strong correlation between the error and its disentangled estimate. This suggests the two forms of errors indeed combine additively without strong interactive effects.

\begin{table*}[htp!]
  \centering
  \small
  \begin{tabular}{llcccccccc}
    \toprule
    Dataset & Method & 0.0 & 0.05 & 0.1 & 0.3 & 0.5 & 0.7 & 0.9 & 0.95 \\
    \midrule
    \multirow{4}{*}{ML-100K} 
    & userKNN     & 1.0123 & 1.0120 & 1.0151 & 1.0231 & 1.0380 & 1.0706 & 1.1675 & 1.2462 \\
    & IGMC        & 0.9281 & 0.9238 & 0.9321 & 0.9604 & 0.9824 & 1.0268 & 1.1165 & 1.1482 \\
    & Soft Impute & 0.9179 & 0.9324 & 0.9373 & 0.9487 & 0.9616 & 1.0380 & 1.4190 & 1.8230 \\
    & DAMC        & \textbf{0.9068} & \textbf{0.9165} & \textbf{0.9241} & \textbf{0.9354} & \textbf{0.9364} & \textbf{0.9621} & \textbf{1.0060} & \textbf{1.0460} \\
    \midrule
    \multirow{4}{*}{Douban}
    & userKNN     & 0.7946 & 0.7948 & 0.7973 & 0.8048 & 0.8288 & 0.8848 & 0.9639 & 0.9838 \\
    & IGMC        & 0.7437 & 0.7480 & 0.7579 & 0.7665 & 0.8010 & 0.8256 & 0.8713 & \textbf{0.8462} \\
    & Soft Impute & 0.7383 & 0.7443 & 0.7427 & 0.7621 & 0.8289 & 1.0145 & 1.8471 & 3.1966 \\
    & DAMC        & \textbf{0.7178} & \textbf{0.7195} & \textbf{0.7205} & \textbf{0.7259} & \textbf{0.7382} & \textbf{0.7618} & \textbf{0.8323} & 0.8577 \\
    \midrule
    \multirow{4}{*}{Yelp}
    & userKNN     & 1.0955 & 1.1020 & 1.1060 & 1.1219 & 1.1496 & 1.2015 & 1.2230 & 1.2356 \\
    & IGMC        & 1.0707 & 1.0955 & 1.0759 & 1.1150 & 1.0899 & 1.1492 & 1.2028 & \textbf{1.0859} \\
    & Soft Impute & 1.3888 & 1.4010 & 1.4126 & 1.7608 & 1.8428 & 2.0162 & 3.2690 & 3.4140 \\
    & DAMC        & \textbf{1.0320} & \textbf{1.0650} & \textbf{1.0470} & \textbf{1.0283} & \textbf{1.0580} & \textbf{1.0750} & \textbf{1.1290} & 1.1390 \\
    \bottomrule
  \end{tabular}
  \caption{Performance in terms of Root Mean Squared Error (RMSE) across three datasets for varying values of $p$. Lower is better. The best results in each setting are presented in  \textbf{boldface}.}
    \label{tab:combined}
\end{table*}

\subsection{Real Data Experiments}

We also perform real data experiments on three popular datasets to evaluate whether the sampling distribution over observed user-item interactions (often referred to as `implicit feedback') contains information which can be used to improve performance at the prediction of \textit{ratings} on a scale from 1 to 5 (often referred to as the `explicit feedback'). We performed experiments on three well-known datasets: Douban~\cite{zhugraphical,zhu2019dtcdr}, Yelp~\cite{zhang2015character} and MovieLens 100K~\cite{harper2015movielens}. As a result of our hypothesis, which concerns semi-supervised matrix completion, our training setup is somewhat different from standard benchmarks: instead of relying on a training set of labeled interactions, we remove a fraction $p$ of the labels associated to interactions in the training set. In other words, a proportion $(1-p)$ of the observations in the training set contain the observed rating, while the remaining interactions are provided merely in the form of user-item interaction pairs $(i,j)$ with no rating information.  Due to the nonlinearity of real-world data, we tested a slight modification of DAMC where the singular value decomposition of the empirical distribution is replaced by a nonlinear autoencoder. However, the inductive matrix completion components (line 3 of Algorithm~\ref{alg:DAMC}) was kept unchanged. We compare the method to various classic baselines for explicit feedback prediction: UserKNN~\cite{herlocker1999algorithmic}, Softimpute~\cite{softimpute}, and IGMC~\cite{IGMC}. We emphasize that our aim is mostly to demonstrate the validity of our learning paradigm, rather than to provide a state-of-the-art recommender systems model: we show that a method which relies on the unsupervised information can provide better predictions on the unseen labeled test data, compared to purely supervised methods which rely only on the fully observed labeled entries (cf. Table~\ref{tab:combined}).

We observe that DAMC, which relies on the unsupervised information in the pure implicit feedback, significantly outperforms classic methods which rely only on the labeled observations in most situations. This indicates that a relationship between the sampling distribution and the ground truth matrix indeed exists, lending legitimacy to our theoretical learning setting.  In particular, DAMC significantly outperform its counterpart Softimpute (which is the same algorithm without using side information) for all nonzero values of $p$.  Most notably, we observe that for large values of $p$ such as $p=0.90$ and $p=0.95$, many of the baseline models relying only on explicit ratings are not able to perform much better than random, whilst the semi-supervised DAMC can still achieve consistently good performance.

\section{Conclusion}

We introduced a new matrix completion learning setting where the sampling distribution and the ground truth matrix are both low-rank and \textit{share common row and column subspaces}. This setting is inspired by the recommender systems application, where unlabeled interactions (`implicit feedback') are usually much more abundant than labeled interactions (`explicit feedback'). Assuming access to a larger amount $M$ of unlabeled samples and a smaller number $N$ of labeled samples, we show generalization error bounds of the form $\widetilde{O}\left(  \sqrt{\frac{[m+n]\erank}{M}} + \sqrt{\frac{d\erank }{N}} +\sqrt{\frac{\nunif[m+n]r}{MN}} \right).$ When either $\nunif$ (the ratio between the maximum and average sampling probability) is $O(d)$ or $N\geq \frac{m+n}{2}$, the higher-order term $\sqrt{\frac{\nunif[m+n]r}{MN}} $ vanishes, demonstrating a disentanglement between two sources of error: the estimation of the shared low-dimensional subspaces relying on the unlabeled samples, and the estimation of the ground truth matrix. In particular, our results demonstrate the ground truth matrix can be recovered accurately even with a vanishingly small number of labeled interactions per row/column. On real data, we show that unlabeled samples can dramatically improve the performance of explicit feedback prediction methods, lending validity to our assumptions. In future work, it would be interesting to distill the real-data results into a true SVD and to tackle the difficulty of removing Assumption~\ref{Assum:gamma} or the uniform marginals assumption by imposing suitable modifications on the algorithm, or to attempt to derive \textit{optimistic bounds} with a fast decay rate in $N$. 

\section*{Acknowledgements}
This research is supported by the National Research Foundation, Singapore under its AI Singapore Programme (AISG Award No: AISG3-PhD-2025-08-066T). Antoine Ledent and Mun Chong Soo's research was supported by the Singapore Ministry of Education (MOE) Academic Research Fund (AcRF) Tier 1 grant.
\bibliography{bibliography}

\onecolumn
\appendix
 \counterwithin{equation}{section}
\renewcommand{\theequation}{\thesection.\arabic{equation}}
\appendix
\title{Supplementary Material for "Generalization Bounds for Matrix Completion with Distributional Side Information"}
\maketitle
\clearpage 

\setcounter{equation}{0}

\renewcommand{\theequation}{A.\arabic{equation}}

\section{Table of notations}
\label{tablenotation}

\begin{center}
	\begin{longtable}{c|c}
		\caption{ Table of notations for quick reference}\\
		Notation & Meaning \\
		\hline
		$\|A\|$ & spectral norm of matrix $A$\\
		$A\leq B$ & $B-A$ is positive semi-definite\\
		$\|A\|_{*}$ & nuclear norm of matrix $A$\\
		$I$ & Identity matrix \\\hline
		$\ground\in\R^{m\times n}$ & ground truth matrix\\
		$\xi_1,\ldots,\xi_N$ &  Observed entries for the \textit{labelled} samples\\
		($\in \{1,\ldots,m\}\times \{1,\ldots,n\}$)  & \\
		$\obvs$ & Observed label for the $o$th observation $(\xi_o, \obvs)$\\
		$\zeta_\xi$ & \makecell{Noise of observed at sample $\xi$\\  i.e. $\obvsnoo=\ground_{\xi}+\noisenoo_{\xi}$}\\
		$X\in \R^{m\times d}$ (resp. $Y\in\R^{n\times d}$) & Row (resp. column) side information matrix\\
		$\midmat$ & matrix to optimize (predictors: $X\midmat Y^\top$)\\
		$S_M=\Omega=\{\xi_1,\ldots,\xi_M\}$ & Observed unlabelled samples\\
		$x_i=X_{i,\nbull}$ & side information vector for $i$th user (row)\\
		$y_j=X_{j,\nbull}$ & side information vector for $j$th item (column)\\
		$\mathbf{x}$  (resp. $\mathbf{y}$) & $\max_i \|x_i\|^2$ (resp. $\max_j \|x_j\|^2$)\\
		$P_{i,j}$ & Probability of sampling $(i,j)$\\
		& =$\P(\xi=(i,j))$\\
		\hline
		$\mathm$ & constraint on $\|\midmat\|_*$ \\
		$\lfn$ & loss function\\
		$b$ & global upper bound on $l$\\\hline
		$\pstar$ & $\pstarexpand$\\
		$\pstarr$ & $\pstarrexpand$ \\
		$\erank$ &   $ \frac{\mathm^2}{d^2}$\\ 
		$\kappa_1$ & $\max(\max_i p_i m,\max_j q_j n)$\\
		$\deltastar$ & $\sigma_{D}$, the last singular value of $P$\\ 
		$\kappa_*$ &  $\frac{\|P\|}{\deltastar}$ \\
		\hline
		$\lip$ & Lipschitz constant of $\lfn$\\
		$\lfn(Z)$ & $\E_{(i,j)\sim p}(\lfn([XMY^\top]_{i,j},G_{i,j}+\zeta_{i,j}))$\\
		$\hat{\lfn}(Z)$ & $\frac{1}{N}\sum_{(i,j)\in \Omega}  \lfn([XMY^\top]_{i,j},G_{i,j}+\zeta_{i,j})$\\
		\hline
		$\rademex$ & $\rademexpand$ \\
		$\prerademex$ & $\prerademexpand$ 
		\label{NOTATIONS}
	\end{longtable}
\end{center}

\clearpage

\appendix

\section{Bounding the Error from the PMF Estimation from Unlabelled Samples. }

\subsection{Bounding the Subspace Recovery Error in Terms of Subspace Distance}

In this section, we aim to prove the following proposition, which ensures that the row and column subspaces recovered from the singular value decomposition of the matrix $\prerademex$ are close to the true subspaces.

\begin{proposition} 
	\label{prop:subspacedistance2}
	Let $UDV^\top$ denote the truncated svd of $\observe$ at rank $d$, where $U\in \R^{m\times d}$ (resp. $V\in \R^{n\times d}$). Denote also  $P=U^{*}D^{*}{V^{*}}^\top$ the svd of $P$. We also write $\deltastar$ for $D^{*}_{d}-D^{*}_{d+1}$.
	W.p. $\geq 1-\delta$ as long     as $ M\geq \log\left(\frac{m+n}{\delta}\right) \left[ \left[\frac{4}{\deltastar [2-\sqrt{2}]} \right]^2\frac{16\pstar}{3}+\frac{32}{\deltastar [2-\sqrt{2}]}\right] $,
	
	\begin{align}
		\label{eq:subspacedistance2}
		&		\max\left(  \min_{R\in\mathcal{O}^{r\times r}}\left\|   UR-U^{*} \right\| , \min_{R\in\mathcal{O}^{r\times r}} \left\|   VR-V^{*} \right\|  \right)  \\& \leq   \frac{2}{\deltastar}  \left[ \sqrt{\frac{16\pstar}{3M}} \sqrt{\log\left(\frac{m+n}{\delta }\right)}+ \frac{16}{M}\log\left(\frac{m+n}{\delta}\right)\right],
	\end{align}	
	where $\pstar:=\pstarexpand$, where $p\in\R^{m}$ (resp.  $q\in\R^{m}$ ) is the vector of marginal row (resp. column) probabilities $p_i=\sum_j P_{i,j}$  (resp.  $q_j=\sum_i P_{i,j}$ ).
\end{proposition}
\begin{remark}
Note that the above Proposition~\ref{prop:subspacedistance2}  (and all the subsequent results) holds even if the rank of $P$ is larger than $d$, as long as there is a significant eigengap between the $d$th and $d+1$th singular values of $P$. Accordingly, although $d$ is presented as the rank of $P$ in the main paper for simplicity, it is straightforward to extend our results to the case where the singular value decomposition in Step 2 of Algorithm~\ref{alg:DAMC} is truncated up to rank $d$, but the true rank of $P$ is higher, removing any need for the user to know the exact rank $d$.  However, Assumptions~\ref{assum:lowrankshared} and~\ref{assum:wellconditioned} must be accordingly modified as follows: 
\begin{itemize}
    \item Assumption~\ref{assum:lowrankshared} merely assumes that $P$ has the \textit{truncated} decomposition $P=U^* \Sigma^*[V^*]^\top$ at rank $d$, and requires the ground truth matrix to be representable as $\ground=U^* \midmatstar [V^*]^\top$ (i.e. the row and column spaces of $\ground$ are contained in the top $d$ eigenspaces of $P$). 
    \item Assumption~\ref{assum:wellconditioned} must define $\deltastar$ as $D_{d}-D_{d+1}$ where $D_d$ and $D_{d+1}$ are the $d$th and $d+1$th singular values of $P$ respectively. 
\end{itemize}
Lastly, we note that the simplified presentation of the main paper which assumes that $d$ is known is a common practice in exact and perturbed recovery in matrix completion. Cf.  Section 3.8.2 (`Algorithm'), page 70 of the celebrated work~\cite{spectralmagic} (the assumption is present in the main theorems on Matrix Completion in this reference). 
\end{remark}

\begin{proof}

	Write $\excess$ for the matrix $\excess=\observe-P=\frac{1}{M} \sum_{o=1}^M (1_{i_o,j_o}-P)=\frac{1}{M} \sum_{o=1}^M \zeta_o$  where $\zeta_o=(1_{i_o,j_o}-P)\in\R^{m\times n}$.

	Note that $\E(\excess)=\E(\zeta_1)=0$. In addition, we have $\|\zeta_1\|\leq 1+\|P\| \leq 2$ w.p. 1. Thus, we have, 
	
	\begin{align}
		\E(\zeta_1 \zeta_1^{\top})&=\E\left(1_{i_1,j_1} 1_{i_1,j_1}^{\top}\right) - \E\left(1_{i_1,j_1}  P^\top \right)-\E\left(  P1_{i_1,j_1}^{\top}\right) + PP^{\top} \\
		&= \diag(p) -PP^\top
	\end{align}
	and 
	\begin{align}
		\E(\zeta_1^{\top} \zeta_1)&=\E\left(1_{i_1,j_1}^{\top} 1_{i_1,j_1}\right) - \E\left(1_{i_1,j_1}^{\top}  P \right)-\E\left(  P^{\top}1_{i_1,j_1}\right) + P^{\top}P \\
		&=\diag(q)-P^\top P.
	\end{align} 
	
	Thus, we can apply Proposition~\ref{bernsteinprob} with constants $(1+\|P\|)/M$ and $$\sum_k\rho_k^2=(\|P\|^2+\max(\max_i p_i,\max_j p_j))/M=\frac{\left[\|P\|^2+\pstar \right]}{M},$$ where $\pstar$ denotes $\max(\max_i p_i,\max_j p_j)$.

	This yields that w.p. $\geq 1-\delta$ 
	\begin{align}
		\left\|\excess\right\|& \leq  \sqrt{\frac{8\left[\|P\|^2+\pstar\right]}{3M}} \sqrt{\log\left(\frac{m+n}{\delta }\right)}+ \frac{8(1+\|P\| )}{M}\log\left(\frac{m+n}{\delta}\right).
	\end{align}

	Thus, applying Proposition~\ref{prop:specralmag}, we have with the same failure probability 
	
	\begin{align}
		\label{eq:subspacedistance1}
		&		\max\left(  \min_{R\in\mathcal{O}^{d\times d}}\left\|   UR-U^{*} \right\| , \min_{R\in\mathcal{O}^{d\times d}} \left\|   VR-V^{*} \right\|  \right)  \\& \leq   \frac{2}{\deltastar}  \left[ \sqrt{\frac{8[\|P\|^2+\pstar]}{3M}} \sqrt{\log\left(\frac{m+n}{\delta }\right)}+ \frac{16}{M}\log\left(\frac{m+n}{\delta}\right)\right].
	\end{align}
	as long as  \textcolor{black}{$ M\geq \log\left(\frac{m+n}{\delta}\right) \left[ \left[\frac{4}{\deltastar [2-\sqrt{2}]} \right]^2\frac{8[\|P\|^2+\pstar]}{3}+\frac{64}{\deltastar [2-\sqrt{2}]}\right] $},
	
	The result now follows by Lemma~\ref{lem:someprogress}.

\end{proof}

\begin{lemma}
	\label{lem:someprogress}
	We have 
	\begin{align}
		\|P\|\leq  \pstar \leq 1 
	\end{align}
	
	
	Further more, we also have 
	\begin{align}
		\label{eq:someprogressconverse}
		\pstar\leq \|P\|\max\left( \sqrt{nd}  \|[U^*]^\top\|_{2,\infty} ,\sqrt{md}  \|[V^*]^\top\|_{2,\infty}    \right)  =\|P\|\max \left( \sqrt{\frac{n}{m}} \mathbf{x},\sqrt{\frac{m}{n}} \mathbf{y}  \right)
	\end{align}
\end{lemma}

\begin{proof}
	First note that we certainly have 
	\[
	|P| \leq \|P\|_{Fr} \leq |P|_1 = 1
	\]
	
	Next, observe that since by definition of $\pstar$,  we have  $	\sum_i P_{ij} \leq \pstar$  and $	\sum_j P_{ij} \leq \pstar$, we certainly have 
	
	\[
	\sum_j (P^T P)_{ij} = \sum_j \sum_k P_{ki} P_{kj} =\sum_k P_{ki}p_k
	\]
	\[
	\leq \sum_k P_{ki} \pstar  \leq (\pstar)^2.
	\]

	Now, note also that 
	\[
	\forall \, \text{unit vector } v, \quad |(P^T P)v|_{\max} \leq (\pstar)^2 |v|_{\max}
	\]
	\[
	\therefore \sigma_1(P)^2 = \lambda_1(P^T P) \leq (\pstar)^2,
	\] 
	which concludes the proof of the first inequality. 
	
	For the second inequality, note that 
	\begin{align}
		P_{i,j}= U^*_{i,\nbull} \Sigma^* [V^*_{j,\nbull}]^\top \leq \|P\|  \|U^*_{i,\nbull}\| \|V^*_{j,\nbull}\| \leq \|[U^*]^\top\|_{2,\infty}\|[V^*]^\top\|_{2,\infty} \|P\|
	\end{align}
	and therefore, for any $i$, 
	\begin{align}
		p_i & =\sum_j P_{i,j} \leq  \|[U^*]^\top\|_{2,\infty} \|P\| \sum_{j}\|V^*_{j,\nbull}\|  \\&   \leq \sqrt{n }  \|[U^*]^\top\|_{2,\infty} \|P\| \|V\|_{\Fr}=\sqrt{nd}  \|[U^*]^\top\|_{2,\infty}  \|P\| \quad \quad \text{and similarly},\\
		q_j &\leq \sqrt{md}  \|[V^*]^\top\|_{2,\infty}\|P\| \quad \quad \quad (\forall j),
	\end{align}
	where in both cases, we have used the AM-GM inequality. 
\end{proof}

We will also need the slightly more general version below:

\begin{lemma}
	\label{lem:someprogresswithsigns}
	
	Let $\widetilde{P}=P\circ S$ where $S\in\{-1,1\}^{m\times n}$ is a sign matrix. 
	
	We have 
	\begin{align}
		\|\widetilde{P}\|\leq \pstar
	\end{align}
\end{lemma}

\begin{proof}
	The proof is very similar to that of Lemma~\ref{lem:someprogress}, being a little more careful with signs. 
	First, note that
	
	\begin{align}
		\left| \left[\widetilde{P}^\top \widetilde{P}\right]_{i,j}\right| \leq [P^\top P]_{i,j}.
	\end{align}

	It follows that for any $i\leq m$, 
	
	\[
	\sum_j \left|  \left[\widetilde{P}^\top \widetilde{P}\right]_{ij}\right|  \leq \sum_j (P^T P)_{ij} \leq \pstar^2 
	\]
	
	Thus, 
	\[
	\forall \, \text{unit vector } v, \quad |(\widetilde{P}^{\top} \widetilde{P})v|_{\max} \leq (\pstar)^2 |v|_{\max}
	\]
	\[
	\therefore \sigma_1(\widetilde{P})^2 = \lambda_1(\widetilde{P}^\top  \widetilde{P}) \leq \pstar^2,
	\] 
	as expected. 
\end{proof}

\subsection{Bounding the Population-level Generalization Error Arising from Subspace Recovery}

 In this subsection, we aim to prove bounds for  effect of the misestimation of $X^*$ and $Y^*$ on the population risk. Notably, this can be done without  Assumption~\ref{Assum:gamma}.
 

\begin{lemma}
	\label{lem:implicit_unif_population}
	As in the definition of the algorithm, let $X=\sqrt{\frac{m}{d}} U$  (resp. $Y=\sqrt{\frac{n}{d}} V$) where $U$ and $V$ are obtained through the SVD of the matrix $\tosvdize:=\tosvdizepand$. Let $\lfn$ be a $\lip$-Lipschitz loss function. Let $R_l\in\R^{d\times d}$ (resp. $R_r\in\R^{d\times d}$) be rotation matrices such that $\|XR_l-X^*\| =\min_{R}\|XR-X^*\|$ and $\|YR_r-Y^*\| =\min_{R}\|YR-Y^*\|$ (in particular, the matrices $R_l$ and $R_r$ are random as they depend on $X,Y$). As long as 
	\begin{align}
		M\geq 470  \log\left(\frac{4[m+n]}{\delta}\right)   \kappa_*^2\pstarr^2[m+n] \label{eq:theconditiononMannoy},
	\end{align} 
	for any $\delta>0$ over the draw of the implicit feedback data $\xi_1,\ldots,\xi_M$, with  probability greater than $1-\delta$,  for any matrices $\midmat_0\in\R^{d\times d}$ and  $\midmat_1:=R_l \midmat_0 R_r^\top$, we have: 
	\begin{align}
		\sum_{i,j} P_{i,j} 	\lfn\left(X\midmat_1Y^\top  -X^*\midmat_0 {Y^*}^{\top}\right) &\leq 120\lip \kappa_* \kappa_1 \pstarr \sqrt{\log\left(\frac{2[m+n]}{\delta }\right)}  \sqrt{\frac{[m+n]\erank }{3M}}  \label{eq:expensivestuffthatis}. 
	\end{align}

\end{lemma}
\begin{proof}
	
	First note that by Proposition~\ref{prop:subspacedistance2}, with probability  $1-\delta$, as long as $ M\geq \log\left(\frac{m+n}{\delta}\right) \left[ \left[\frac{4}{\deltastar [2-\sqrt{2}]} \right]^2\frac{16\pstar}{3}+\frac{64}{\deltastar [2-\sqrt{2}]}\right] $,

	\begin{align}
		\label{eq:neweee}
		&		 \left\|   XR_l-X^{*} \right\|  \leq   \frac{2\sqrt{\frac{m}{d}}}{\deltastar}  \left[ \sqrt{\frac{16\pstar}{3M}} \sqrt{\log\left(\frac{m+n}{\delta }\right)}+ \frac{16}{M}\log\left(\frac{m+n}{\delta}\right)\right]\nonumber \\
		&\leq \frac{4\sqrt{m}}{\deltastar} \sqrt{\frac{16\pstar}{3Md}} \sqrt{\log\left(\frac{m+n}{\delta }\right)},
	\end{align}	
	where at the last line we have used the condition on $M$. 
	
	Note that we can further bound 
	\begin{align}
		&\frac{4\sqrt{m}}{\deltastar} \sqrt{\frac{16\pstar}{3Md}} \sqrt{\log\left(\frac{m+n}{\delta }\right)}\leq 16  \sqrt{\frac{\kappa_*^2  \pstarr^2 m/\pstar }{3Md}} \sqrt{\log\left(\frac{m+n}{\delta }\right)} \nonumber \\
		&\leq \kappa_* \pstarr 16  \sqrt{\frac{mn}{3Md}} \sqrt{\log\left(\frac{m+n}{\delta }\right)}  \label{eq:roundandroundtheincoherence}
	\end{align}
	where $\pstarr$ is defined as the dimension-free quantity $\pstarr:=\pstarrexpand\geq \frac{\pstar}{\|P\|}$.

	From this we immediately obtain, with probability $\geq 1-\delta/2$
	\begin{align}
		\label{eq:spectrarivalleft}
		\left\|   XR_l-X^{*} \right\|_{} \leq  \kappa_* \pstarr 16  \sqrt{\frac{mn}{3dM}} \sqrt{\log\left(\frac{2[m+n]}{\delta }\right)} :=\bd.
	\end{align}
	Similarly, we also have 
	\begin{align}
		\label{eq:spectrarivalright}
		\left\|   YR_r-R^{*} \right\|_{} \leq  \kappa_* \pstarr 16  \sqrt{\frac{mn}{3dM}} \sqrt{\log\left(\frac{2[m+n]}{\delta }\right)} =\bd,
	\end{align}
	and by a union bound, both of the inequalities above hold simultaneously with a total failure probability less than $\delta$.

	Note we have the following decomposition: 
	\begin{align}
		&X\midmat_1Y^\top  -X^*\midmat_0 {Y^*}^{\top}\\&=	XR_l \midmat_0 R_r^\top Y^\top  -X^*\midmat_0 {Y^*}^{\top} \nonumber  \\
		&=[XR_l -X^*] \midmat_0 R_r^\top Y^\top +X^* \midmat_0 [YR_r -Y^*]^\top \nonumber \\
		&=[XR_l -X^*] \midmat_0 [YR_r -Y^*]^\top  +X^* \midmat_0 [YR_r -Y^*]^\top +[XR_l -X^*] \midmat_0 [Y^*]^\top.\label{eq:thedecomposition}
	\end{align}

	Then, from equation ~\eqref{eq:spectrarivalleft} we obtain 
	\begin{align}
		\left\langle \left|  [XR_l -X^*] \midmat_0 [Y^*]^\top \right| , P\right\rangle_{\Fr} &= 	\left\langle  [XR_l -X^*] \midmat_0 [Y^*]^\top, \widetilde{P}\right\rangle_{\Fr}   \label{eq:newneed}\\ 
		&\leq \left\|  [XR_l -X^*] \midmat_0 [Y^*]^\top\right\|_* \|\widetilde{P}\|\\
		&\leq  \|  [XR_l -X^*] \| \|\midmat_0 \|_* \|[Y^*]^\top\|   \|\widetilde{P}\|\\
		&\leq      \|  [XR_l -X^*] \| \|\midmat_0 \|_* \|[Y^*]^\top\|     p^*                                                              \label{eq:usespectralP}    \\
		&\leq   \|  [XR_l -X^*]  \|      p^*  \sqrt{\frac{n}{d}}   \mathm                  \label{eq:usemathmmmm}\\
		& \leq   \kappa_* \pstarr 16  \sqrt{\frac{mn}{3dM}} \sqrt{\log\left(\frac{m+n}{\delta }\right)}           p^*  \sqrt{\frac{n}{d}}   \mathm                          \label{eq:finallyuserealstuff} \\
		&\leq       16 \kappa_* \kappa_1 \pstarr    \sqrt{\log\left(\frac{m+n}{\delta }\right)}  \sqrt{\frac{[m+n]\erank }{3M}},                 \label{eq:leftfinal}
	\end{align}
	where at the first line~\eqref{eq:newneed} we have defined $\widetilde{P}=P\circ S$ for $\{-1,1\}^{m\times n} \ni S=\left( [XR_l -X^*] \midmat_0 [Y^*]^\top \right)$,  at Line~\eqref{eq:usespectralP}, we have used Lemma~\ref{lem:someprogresswithsigns}, at Line~\eqref{eq:usemathmmmm} we have used the assumption that $\|\midmat_0\|_*\leq \mathm$ and the fact that $Y^*=U^* \sqrt{\frac{n}{d}}$ (recall that $\|U^*\|=1$),  at Line~\eqref{eq:finallyuserealstuff} we have used Line~\eqref{eq:spectrarivalleft}, at Line~\eqref{eq:leftfinal} we have used the fact that $\pstar\leq \kappa_1 \frac{1}{\min(m+n)}$ and used the notation $\erank$ for the quantity $\erank:=\frac{\mathm^2}{d^2}$.

	With a nearly identical calculation relying on equation~\eqref{eq:spectrarivalright} instead of equation~\eqref{eq:spectrarivalleft}, we also have~\footnote{$\widetilde{P}=P\circ S$ with $\{-1,1\}^{m\times n} \ni S=\sign\left( X^* \midmat_0 [YR_r -Y^*]^\top\right)$}
	\begin{align}
		\left \langle \left|X^* \midmat_0 [YR_r -Y^*]^\top \right| ,P\right\rangle_{\Fr}\leq  16 \kappa_* \kappa_1 \pstarr    \sqrt{\log\left(\frac{m+n}{\delta }\right)}  \sqrt{\frac{[m+n]\erank }{3M}}.\label{eq:rightfinal}
	\end{align}

	Next, to handle the cross term in equation~\eqref{eq:thedecomposition}, note that we also have through a similar calculation~\footnote{$\widetilde{P}=P\circ S$ with $\{-1,1\}^{m\times n} \ni S=\sign\left( [XR_l -X^*] \midmat_0 [YR_r -Y^*]^\top\right)$ }
	\begin{align}
		\left\langle \left| [XR_l -X^*] \midmat_0 [YR_r -Y^*]^\top\right|, P\right\rangle_{\Fr} &\leq  \left\|  [XR_l -X^*] \midmat_0 [YR_r -Y^*]^\top   \right\|_* \|\widetilde{P}\|    \\
		&\leq  \|  [XR_l -X^*]\|  \| [YR_r -Y^*]^\top\|   \pstar \mathm   \label{eq:usespectraPagain}    \\
		& \leq \bd^2 \pstar \mathm   \label{eq:finallyuserealstuffagain}\\
		& \leq  256 \kappa_*^2 \pstarr^2  \log\left(\frac{m+n}{\delta }\right)  \frac{mn}{3dM} \pstar \mathm \\
		&\leq   256 \kappa_*^2 \pstarr^2  \log\left(\frac{m+n}{\delta }\right)     \sqrt{\erank} \frac{mn}{3M} \pstar                                      \label{eq:usedefiner} \\
		& \leq        86  \kappa_1  \kappa_*^2 \pstarr^2  \log\left(\frac{m+n}{\delta }\right) \frac{[m+n]\sqrt{\erank}}{M},            	   \label{eq:crossfinal}
	\end{align}
	where at Line~\eqref{eq:usespectraPagain} we have used Lemma~\ref{lem:someprogresswithsigns} and  the assumption that $\|\midmat_0\|_*\leq \mathm$, at Line~\eqref{eq:finallyuserealstuffagain} we have used equations~\eqref{eq:spectrarivalleft} and~\eqref{eq:spectrarivalright}, and at Line~\eqref{eq:usedefiner} we have used the definition $\erank:=\frac{\mathm^2}{d^2}$ and at the last line~\eqref{eq:crossfinal} we have used the fact that $\pstar\leq \kappa_1 \frac{1}{\min(m,n)}$. 
	
	We are now in a position to prove inequality~\eqref{eq:expensivestuffthatis}. Indeed, we have: 
	\begin{align}
		&	\sum_{i,j} P_{i,j} 	\lfn\left(X\midmat_1Y^\top  -X^*\midmat_0 {Y^*}^{\top}\right)_{i,j} \leq \lip \left\langle      \left|X\midmat_1Y^\top  -X^*\midmat_0 {Y^*}^{\top}\right|,P\right \rangle_{\Fr} \label{eq:lippppp} \\
		&\leq   \lip  \Bigg[\left\langle [XR_l -X^*] \midmat_0 [YR_r -Y^*]^\top,P\right\rangle_{\Fr}  +\left\langle X^* \midmat_0 [YR_r -Y^*]^\top,P\right\rangle_{\Fr}  \\&  \quad \quad \quad \quad \quad \quad +\left\langle [XR_l -X^*] \midmat_0 [Y^*]^\top,P\right\rangle_{\Fr}       \Bigg] \label{eq:thedecompositionisusedfinally} \\
		&\leq   86 \lip \kappa_1  \kappa_*^2 \pstarr^2  \log\left(\frac{m+n}{\delta }\right) \frac{[m+n]\sqrt{\erank}}{M}+ 32 \lip \kappa_* \kappa_1 \pstarr    \sqrt{\log\left(\frac{m+n}{\delta }\right)}  \sqrt{\frac{[m+n]\erank }{3M}}, \label{eq:throwitallin} 
	\end{align}
	where at Line~\eqref{eq:lippppp} we have used the Lipschitz condition on the loss function; at Line~\eqref{eq:thedecompositionisusedfinally} we have used the decomposition~\eqref{eq:thedecomposition}; at Line~\eqref{eq:throwitallin} we have used equations~\eqref{eq:crossfinal},~\eqref{eq:rightfinal} and~\eqref{eq:leftfinal}. This is subject to $M$ exceeding the LHS of the following equation as a \textcolor{black}{threshold}:  
	\begin{align}
		& \log\left(\frac{2[m+n]}{\delta}\right) \left[ \left[\frac{4}{\deltastar [2-\sqrt{2}]} \right]^2\frac{16\pstar}{3}+\frac{64}{\deltastar [2-\sqrt{2}]}\right] \\ &\leq \log\left(\frac{2[m+n]}{\delta}\right)  \left[ 341 \kappa_*^2\pstarr^2[m+n] +  128 \kappa_* \pstarr [m+n]     \right]\nonumber \\
		&\leq 470  \log\left(\frac{2[m+n]}{\delta}\right)   \kappa_*^2\pstarr^2[m+n],
	\end{align}
	where we have used a calculation similar to that from equation~\eqref{eq:roundandroundtheincoherence} to handle the first term and used the fact that $\frac{1}{\|P\|}\leq \pstarr/\pstar\leq \pstarr [m+n]$ to handle the second term.  Plugging this back into equation~\eqref{eq:throwitallin}, we obtain the final result.

\end{proof}

\subsection{Bounding the Error from Subspace Recovery at the Empirical Level}

\begin{proposition}
	\label{prop:controlempirical}
	Assuming 
	\begin{align}
		\label{cond:hereagain}
		M\geq 470  \log\left(\frac{4[m+n]}{\delta}\right)   \kappa_*^2\pstarr^2[m+n],
	\end{align}
	we have $	\frac{1}{N} \left[	\lfn\left(X\midmat_1Y^\top,\obvsnoo)\right)_{\xi_o}   -\lfn\left(X^*\midmat_0 {Y^*}^{\top},\obvsnoo\right)_{\xi_o}\right]\leq$
	\begin{align}
		& 4\entry \frac{\log(2/\delta)}{3N}  +25   \pstarr \ell  \log\left(\frac{4[m+n]}{\delta }\right)  \left[3 \kappa_1\sqrt{\frac{[m+n]\erank}{M}}+\sqrt{\frac{[m+n]\nunif \erank}{MN}}\right] \\&\leq  4\entry \frac{\log(2/\delta)}{3N}  +25   \pstarr \ell  \log\left(\frac{4[m+n]}{\delta }\right)  \left[3 \kappa_1\sqrt{\frac{[m+n]\erank}{M}}+\sqrt{\kappa_1\frac{[m+n]^2\erank}{MN}}\right].
	\end{align}
\end{proposition}
\begin{proof}
	The strategy is to use the noncentered Bernstein inequality to estimate the concentration of the empirical quantity $\sum_{o=1}^N \frac{1}{N} 	\lfn\left(X\midmat_1Y^\top  -X^*\midmat_0 {Y^*}^{\top}\right)_{\xi_o} $. For that effect, note that this quantity is an average of $N$ random variables $T_o$ distributed according to $\lfn(X\midmat_1Y^\top -X^*\midmat_0{Y^*}^\top,\obvsnoo)$. We will apply the noncentered Bernstein inequality~\ref{prop:noncentered} to provide a high probability bound. To do this, note first that by Lemma~\ref{lem:implicit_unif_population}, for any $\delta$, the expectation is bounded w.p. $1-\delta$ as
	\begin{align}
		\E(T_o) \leq 120\lip \kappa_* \kappa_1 \pstarr \sqrt{\log\left(\frac{2[m+n]}{\delta }\right)}  \sqrt{\frac{[m+n]\erank }{3M}}
	\end{align}
	
	Furthermore, we clearly have 
	\begin{align}
		\E(T_o)\leq 2 \entry. 
	\end{align}
	
	Next, we must calculate the second moment: 
	\begin{align}
	& 	|	\E(T_o^2)| \\&\leq \lip^2  \sum_{i,j} P_{i,j}\left[X\midmat_1Y^\top  -X^*\midmat_0 {Y^*}^{\top}\right]_{i,j}^2 \\
		&\leq 3\lip^2 \sum_{i,j} P_{i,j} \Bigg[ \left[ X^*\midmat_0 [YR_r -Y^*]^\top  \right]_{i,j}^2 +  \left[[XR_l -X^*] \midmat_0 [YR_r -Y^*]^\top\right]_{i,j}^2   \\ & \quad \quad \quad \quad \quad \quad  +\left[ [XR_l -X^*] \midmat_0 [Y^*]^\top  \right]_{i,j}^2       \Bigg] \\
		&\leq 3\frac{\lip^2\nunif}{mn}  \Bigg[ \left\| X^*\midmat_0 [YR_r -Y^*]^\top  \right\|_{\Fr}^2+  \left\|[XR_l -X^*] \midmat_0 [YR_r -Y^*]^\top\right\|_{\Fr}^2   \\& \quad \quad \quad \quad \quad \quad   +  \left\| [XR_l -X^*] \midmat_0 [Y^*]^\top  \right\|_{\Fr}^2       \Bigg] \label{eq:pointer}
	\end{align}
	
	Now recall from equations~\eqref{eq:spectrarivalleft} and~\eqref{eq:spectrarivalright} that w.p. $\geq 1-\delta/2$, 
	\begin{align}
		\left\|   YR_r-Y^{*} \right\|_{}, \left\|   XR_l-X^{*} \right\|_{} \leq  \kappa_* \pstarr 16  \sqrt{\frac{mn}{3dM}} \sqrt{\log\left(\frac{2[m+n]}{\delta }\right)} :=\bd.
	\end{align}

	Thus we can continue from equation~\eqref{eq:pointer}:

	\begin{align}
		\E(T_o^2)&\leq 3\frac{\lip^2\nunif}{mn}  \Bigg[ \left\| X^*\midmat_0 [YR_r -Y^*]^\top  \right\|_{\Fr}^2+  \left\|[XR_l -X^*] \midmat_0 [YR_r -Y^*]^\top\right\|_{\Fr}^2   \\ & \quad \quad \quad \quad \quad \quad \quad \quad \quad\quad \quad \quad \quad \quad \quad \quad \quad \quad \quad \quad \quad \quad +  \left\| [XR_l -X^*] \midmat_0 [Y^*]^\top  \right\|_{\Fr}^2       \Bigg] \\
		&\leq  3\frac{\lip^2\nunif}{mn}  \left[  \erank d^2\frac{m+n}{d}  \bd^2 +\bd^4  \erank d^2 \right]\\
		&\leq 256 \kappa_*^2 \pstarr^2 \nunif \lip^2 \log\left(\frac{2[m+n]}{\delta }\right) \left[ \frac{[m+n]\erank}{M} + 86\kappa_*^2\pstarr^2\frac{mn \erank }{M^2}  \log\left(\frac{2[m+n]}{\delta }\right)\right] \\
		&\label{eq:useusecondcond}\leq  303   \kappa_*^2 \pstarr^2 \nunif \lip^2 \log\left(\frac{2[m+n]}{\delta }\right) \frac{[m+n]\erank}{M},
	\end{align}
	
	where at line~\eqref{eq:useusecondcond} we have used the condition~\eqref{cond:hereagain}. 
	
	We are now in a position to apply Proposition~\ref{prop:noncentered} with $\delta\leftarrow \delta/2$ (using the same value of delta in our use of equation~\eqref{eq:useusecondcond}) to conclude that with overall failure probability $\leq \delta$ we have 
	
	\begin{align}
	&	\frac{1}{N} \left[	\lfn\left(X\midmat_1Y^\top,\obvsnoo)\right)_{\xi_o}   -\lfn\left(X^*\midmat_0 {Y^*}^{\top},\obvsnoo\right)_{\xi_o}\right]\\&\leq 4\entry \frac{\log(2/\delta)}{3N} +25  \kappa_* \pstarr  \lip \sqrt{ \nunif \log\left(\frac{4[m+n]}{\delta }\right) \frac{[m+n]\erank}{MN} \log(2/\delta )  }\\
		&\quadfive \quad \quad +120\lip \kappa_* \kappa_1 \pstarr \sqrt{\log\left(\frac{2[m+n]}{\delta }\right)}  \sqrt{\frac{[m+n]\erank }{3M}}\\
		&\leq 4\entry \frac{\log(2/\delta)}{3N}  +25   \pstarr \ell  \log\left(\frac{4[m+n]}{\delta }\right)  \left[\sqrt{\kappa_1\frac{[m+n]^2\erank}{MN}}+3 \kappa_1\sqrt{\frac{[m+n]\erank}{M}}\right],
	\end{align}
	where we have used the fact that $\nunif\leq \kappa_1\min(m,n)$.
\end{proof}

\section{Proof of the Main Results }

Armed with the above, we can finally, proving Theorem~\ref{thm:unif_generalization_bound_new}, which is reproduced in this appendix for completeness.

\begin{theorem}
	\label{thm:unif_generalization_bound_new}
	
	Let the implicit and explicit feedback be drawn as described above and let the assumptions from subsection~\ref{sec:assum} hold. Assume also as in Corollary~\ref{cor:improveantoine} that inequalities~\eqref{eq:kappa2cond} hold. Assume also as usual  that the loss function is upper bounded by $\entry$.
	
	Assume that 
	
	\begin{align}
		\label{cond:Magain??}
		M\geq 470  \log\left(\frac{4[m+n]}{\delta}\right)   \kappa_*^2\pstarr^2[m+n].
	\end{align}
	
	With probability greater than $1-\delta$ over the draw of both the implicit and explicit feedbacks,  the following generalization bound  holds simultaneously over any predictor  $X\midmat Y^\top\in\R^{m\times n}$  for $\midmat\in\R^{d\times d}$ such that $\|\midmat\|\leq \mathm$
	\begin{align}
		\lfn \left(X\midmat Y^\top \right)& \leq   \lfnhat\left( X \midmat Y^\top  \right)   +  2.5\frac{\entry \log(6/\delta )}{\sqrt{N}}+16\lip \pstarr^2 \sqrt{\kappa_1\kappa_2} \log(2de) \sqrt{\frac{d\erank }{N}}\nonumber \\
		&\quadfive+ 25   \pstarr \ell  \kappa_*\log\left(\frac{12[m+n]}{\delta }\right)  \left[3 \kappa_1\sqrt{\frac{[m+n]\erank}{M}}+\sqrt{\frac{[m+n]\nunif \erank}{MN}}\right], 
	\end{align}
	
	where as usual, $\lfn(Z):=\E\left( \lfn(Z_{\xi},\obvsnoo)\right)$  and $\lfnhat(Z):=\frac{1}{N} \sum_{o=1}^N \lfn(Z_{\xi^e_o},\obvs)$. 
	
\end{theorem}

\begin{proof}
	
	As in Lemmas~\ref{lem:implicit_unif_population}, we use the notation $R_l\in\R^{d\times d}$ (resp. $R_r\in\R^{d\times d}$) for the rotation matrices such that $\|XR_l-X^*\| =\min_{R}\|XR-X^*\|$ and $\|YR_r-Y^*\| =\min_{R}\|YR-Y^*\|$. We also define $\midmatprime:=\midmatprimepand$  In particular, the matrices $R_l$ and $R_r$ are random as they depend on $X,Y$, and so is the matrix $\midmathatprime$ (but Lemma~\ref{lem:implicit_unif_population} and Proposition~\ref{prop:controlempirical} show it has favorable properties with high probability).

	With this definition, we have with probability greater than $1-\delta$: 
	\begin{align}
		&\lfn(X\midmat Y^\top ) \\ &\leq \lfn\left( \stx \midmatprime \sty  \right)   +   120\lip \kappa_* \kappa_1 \pstarr \sqrt{\log\left(\frac{6[m+n]}{\delta }\right)}  \sqrt{\frac{[m+n]\erank }{3M}}       \label{eq:use_implicit_unif_population_new}\\
		&\leq   \lfnhat\left( \stx \midmatprime \sty  \right)   +   120\lip \kappa_* \kappa_1 \pstarr \sqrt{\log\left(\frac{6[m+n]}{\delta }\right)}  \sqrt{\frac{[m+n]\erank }{3M}} \nonumber \\  & \quadfive+   4\ell\mathbf{x}\mathbf{y}\frac{ \mathm}{\sqrt{Nd}}\sqrt{\kappa_1\kappa_2}(1+\sqrt{\log(2d)}) + \frac{6\ell}{N}\mathm\mathbf{x}\mathbf{y}(1+\log(2d))+\entry \sqrt{\frac{\log(6/\delta)}{2N}}    \label{eq:use_improveantoine_new}  \\
		&	 \leq \lfnhat\left( X \midmat Y^\top  \right)   +   120\lip \kappa_* \kappa_1 \pstarr \sqrt{\log\left(\frac{6[m+n]}{\delta }\right)}  \sqrt{\frac{[m+n]\erank }{3M}}    \nonumber \\&
		\quadfive+   4\ell\mathbf{x}\mathbf{y}\frac{ \mathm}{\sqrt{Nd}}\sqrt{\kappa_1\kappa_2}(1+\sqrt{\log(2d)}) + \frac{6\ell}{N}\mathm\mathbf{x}\mathbf{y}(1+\log(2d)) +\entry \sqrt{\frac{\log(6/\delta)}{2N}} \nonumber \\ 	\label{eq:use_implicit_unif_empirical_new}   
		&\quadfive  \quad	+ 4\entry \frac{\log(6/\delta)}{3N}  +25   \pstarr \ell  \log\left(\frac{12[m+n]}{\delta }\right)  \left[3 \kappa_1\sqrt{\frac{[m+n]\erank}{M}}+\sqrt{\frac{[m+n]\nunif \erank}{MN}}\right]   \\
		&\leq    \lfnhat\left( X \midmat Y^\top  \right)   +  2.5\frac{\entry \log(6/\delta )}{\sqrt{N}}+8\lip \pstarr^2 \sqrt{\kappa_1\kappa_2} \log(2de) \sqrt{\frac{d\erank }{N}}\left[1+\sqrt{\frac{1}{N}}\right] \nonumber \\
		&\quadfive+ 25   \pstarr \ell \kappa_* \log\left(\frac{12[m+n]}{\delta }\right)  \left[3 \kappa_1\sqrt{\frac{[m+n]\erank}{M}}+\sqrt{\frac{[m+n]\nunif \erank}{MN}}\right]  \label{eq:finalgeneralizationbound_new}
	\end{align}
	where at Line~\eqref{eq:use_implicit_unif_population_new} we have used Lemma~\ref{lem:implicit_unif_population} (with $\delta\leftarrow \delta/3$), at Line~\eqref{eq:use_improveantoine_new} we have used Corollary~\ref{cor:improveantoine} (with $\delta\leftarrow \delta/3$)  and at Line~\eqref{eq:use_implicit_unif_empirical_new} we have used Proposition~\ref{prop:controlempirical} (with $\delta\leftarrow \delta/3$), and at the last line~\eqref{eq:finalgeneralizationbound_new}, we have used the  condition~\eqref{cond:Magain??} in calculations.
\end{proof}

Finally, we can also express the result as an excess risk bound as follows.

\begin{corollary}
	\label{cor:unimargfinal_new}
	
	Assume that the loss function is bounded by $\entry$ and $\lip$-Lipschitz. 
	
	Let $\widehat{Z}$ be the matrix output by algorithm~\ref{alg:DAMC}, we have the following excess risk bound with probaility $\geq 1-\delta$
	
	\begin{align}
		\E(\lfn(\widehat{Z}, \obvsnoo))&\leq \widehat{\E}(\lfn(\widehat{Z}, \obvsnoo)) + O\left[  [ \lip+\entry]  \kappa_* \pstarr\kappa_1\log\left(\frac{[m+n]}{\delta }\right)\sqrt{\frac{[m+n]\erank}{M}} \left[   1+\sqrt{\frac{\nunif}{N}} \right]\right]   \nonumber \\
		&\quadfive+ O\left[\lip \pstarr^2 \sqrt{\kappa_1\kappa_2} \log(2de) \sqrt{\frac{d\erank }{N}}  \right].
	\end{align}
	In particular, since Assumption~\ref{assum:unimarg} implies Assumption~\ref{Assum:gamma} with $\Gamma\leq \kappa_1 [m+n]$, as long as 
	\begin{align}
		N\geq \frac{m+n}{2}, \label{cond:itchyap}
	\end{align}
	the following bound holds without imposing Assumption~\ref{Assum:gamma}: 
	
	\begin{align}
		\E(\lfn(\widehat{Z}, \obvsnoo))- \widehat{\E}(\lfn(\widehat{Z}, \obvsnoo))  &\leq \widetilde{O}\left[   [\lip+\entry]  \kappa_* \pstarr\kappa_1^{3/2}\sqrt{\frac{[m+n]\erank}{M}} +  \lip \pstarr^2 \sqrt{\kappa_1\kappa_2} \sqrt{\frac{d\erank }{N}}\right],
	\end{align}
	where the notation $\widetilde{O}$ hides polylogarithmic factors in $m,n,\delta$. Likewise, the above also holds if Assumption~\ref{Assum:gamma} is assumed to hold for a value of $\gamma$ which grows linearly in $d$ (but not $m,n$). 
\end{corollary}
\begin{proof}
	This follows immediately from Theorem~\ref{thm:unif_generalization_bound_new} together with the classic lemma of Learning theory. The condition on $M$ is absorbed into the O notation since the LHS is always bounded by $2\entry$  and hence the bound trivially holds if $M<470  \log\left(\frac{4[m+n]}{\delta}\right)   \kappa_*^2\pstarr^2[m+n].$ 
\end{proof}

\section{Matrix Completion Result}

In this section, we recall some classic generalization bounds for inductive matrix completion in the i.i.d. setting which we will then use to control the excess risk in our own results. We slightly modify the final form of the results, simplifying and generalizing the assumptions. 

We recall the following result from~\citet{LedentIMC} (Proposition 3.2):

\begin{proposition}[Proposition 3.2 in~\citet{LedentIMC}]
	\label{prop:uniformgeneral}
	Let $X\in\R^{m\times d}, Y\in\R^{n\times d}$ be two side information matrices and let $\widehat{Z}= X\widehat{\midmat} Y^\top$ where $\widehat{\midmat} $ denotes the solution to the following optimization problem
	\begin{align}
		\midmathat =	\argmin_{\|\midmat\|_*\leq \mathm}  \widehat{\E} \lfn \left([X\midmat Y^\top]_{\xi_o},\obvs  \right).
	\end{align}

	W.p. $\geq 1-\delta$ over the draw of the training set $S$ we have: 
	\begin{align}
		\label{arbitrarp}
		l(\widehat{Z})-l(\ground)&\leq \frac{8\ell}{\sqrt{N}}\mathm\max(\sigma^1_*,\sigma^2_*)(1+\sqrt{\log(2d)}) + \frac{12\ell}{N}\mathm\mathbf{x}\mathbf{y}(1+\log(2d))+b\sqrt{\frac{\log(2/\delta)}{2N}},
	\end{align}
	where $\sigma^*_1$ and  $\sigma^*_2$ are defined by 
	\begin{align}
		\sigma^*_1:=\left[\|X^\top \diag(\kappa^l) X\|\right]^{1/2} \quad \quad \text{and} \quad \quad 	\sigma^*_2:=\left[\|Y^\top \diag(\kappa^r) Y\|\right]^{1/2}
	\end{align}
	with 
	\begin{align}
		\kappa^l_i =\sum_{j=1}^n p_{i,j} \|Y_{j,\nbull}\|^2 \quad \quad \text{and} \quad \quad 
		\kappa^r_i =\sum_{i=1}^m p_{i,j} \|X_{i,\nbull}\|^2.
	\end{align}
\end{proposition}

Using the above, we get the following corollary, which is a very slight modification of Proposition 3.1 in~\cite{LedentIMC}. 

\begin{corollary} [Cf. Propositions 3.1 and 3.2 in~\cite{LedentIMC}]
	\label{cor:improveantoine}
	Let $\mathbf{x}=\|X^\top\|_{2,\infty} =\max_i \|X_{i,\nbull}\|$, and  $\mathbf{y}=\|Y^\top\|_{2,\infty}=\max_j \|Y_{j,\nbull}\|$. Assume that the marginal sampling distributions satisfy the following near uniformity conditions: 
	\begin{align}
		\label{eq:updatedconds}
		p_i\leq \frac{\kappa_1}{m}  \andd q_j \leq \frac{\kappa_1}{n}.
	\end{align}
	
	Assume in addition that the side information matrices satisfy: 
	\begin{align}
		\label{eq:kappa2cond}
		\|X\|\leq \mathbf{x}\sqrt{\kappa_2\frac{m}{d}} \andd 	 	\|Y\|\leq\mathbf{x} \sqrt{\kappa_2\frac{n}{d}},
	\end{align}
	for some $\kappa_2$. Then we have w.p. $\geq 1-\delta$, for any $\R^{m\times n}\ni Z=X\midmat Y^\top $,
	\begin{align}
		\left|\lfn(Z)-\hat{\lfn}(Z) \right| \leq 4\ell\mathbf{x}\mathbf{y}\frac{ \mathm}{\sqrt{Nd}}\sqrt{\kappa_1\kappa_2}(1+\sqrt{\log(2d)}) + \frac{6\ell}{N}\mathm\mathbf{x}\mathbf{y}(1+\log(2d))+b\sqrt{\frac{\log(2/\delta)}{2N}}.	 	
	\end{align}
	In particular,  we have the following excess risk bound for the empirical risk minimizer $\widehat{Z}=X\widehat{\midmat}Y^{\top}$: 
	\begin{align}
		l(\widehat{Z})-l(\ground) \leq 8\ell\mathbf{x}\mathbf{y}\frac{ \mathm}{\sqrt{Nd}}\sqrt{\kappa_1\kappa_2}(1+\sqrt{\log(2d)}) + \frac{12\ell}{N}\mathm\mathbf{x}\mathbf{y}(1+\log(2d))+b\sqrt{\frac{2\log(2/\delta)}{N}}.
	\end{align}

\end{corollary}

\textbf{Remarks:} Compared the conditions in Corollary~\ref{cor:improveantoine} to its parent result Proposition 3.1 in~\citet{LedentIMC}, we require a slightly different condition on the marginals. Conditions~\eqref{eq:updatedconds} are `softer' than the corresponding constraint in equation (9) in~\citet{LedentIMC} since we only require inequality rather than equality (approximately uniform marginals, rather than uniform marginals). On the other hand, equation (9) in~\citet{LedentIMC} only requires uniformity `with respect to the side information matrices $X,Y$', which can be a weaker condition. For our purposes, we find Conditions~\eqref{eq:updatedconds} more intuitive and sufficient.

If we interpret the result in terms of sample complexity, the conclusion (as in~\citet{LedentIMC}) is that the number of required samples is $\widetilde{O}(dr)$ where $r$ is the rank of the ground truth $\midmatstar$ (which is also the rank of $\ground$). Indeed, if the entries of $\midmatstar$ are $O(1)$ the constraint $\mathm$ can be understood to scale like $d\sqrt{r}$ for `most' such matrices. Indeed, $\|\midmatstar\|_*\leq \sqrt{r}\|\midmatstar\|_{\Fr}=\sqrt{r} \sqrt{d^2}$. To interpret the bounds, it is best to think of $\mathbf{x},\mathbf{y}, \ell,b$ as scaling constraints of $O(1)$. Note that (ignoring logarithmic factors) it is reasonable to assume that such scaling results in the entries of $X\midmatstar Y^{\top}$ being $O(1)$. Indeed, if the matrix $\midmatstar$ is sampled from a prior i.i.d. Gaussian distribution, then each entry $(X\midmatstar Y^\top)_{i,j}$  has variance $\|X_{i,\nbull}\|^2\|Y_{j,\nbull}\|^2\leq \mathbf{x}^2\mathbf{y}^2$ and the maximum entry is bounded with high probability by $O(\log(mn) \mathbf{x}\mathbf{y})$. The conclusion doesn't change (up to constants) if $\midmatstar$ is sampled from a low rank prior such as $AB^\top$ with the entries of $A,B$ being i.i.d. Gaussians. 

\begin{proof}
	
	The proof is a simple application of Proposition~\ref{prop:uniformgeneral} after calculating $\sigma^*_1$ and  $\sigma^*_2$. In fact, note first that we have for any $i$: 
	\begin{align}
		\kappa^l_i =\sum_{j=1}^n p_{i,j} \|Y_{j,\nbull}\|^2 \leq \mathbf{y} \sum_{j=1}^n p_{i,j}=\mathbf{y}p_i \leq \kappa_1 \frac{\mathbf{y}}{m}, 
	\end{align}
	and similarly, for any $j$, 
	\begin{align}
		\kappa^r_j\leq \kappa_1 \frac{\mathbf{x}}{n}
	\end{align}
	
	Thus, we can write: 
	\begin{align}
		[  	\sigma^*_1]^2&= \|X^\top \diag(\kappa^l) X\| \leq \|X\|^2 \|\diag(\kappa^l)\| \leq\mathbf{x}^2 \kappa_2 \frac{m}{d}\frac{\mathbf{y}^2\kappa_1 }{m} =\mathbf{x}^2 \mathbf{y}^2\frac{\kappa_1 \kappa_2}{d}, \quad \text{and similarly, }\nonumber \\
		[  	\sigma^*_2]^2&\leq \mathbf{y}^2 \kappa_2 \frac{n}{d}\frac{\mathbf{x}^2\kappa_1}{m} =\mathbf{x}^2 \mathbf{y}^2\frac{ \kappa_1 \kappa_2}{d}
	\end{align}
	Plugging this into equation~\eqref{arbitrarp} yields the result immediately.  
	
\end{proof}
\section{Further Experimental Details}
\subsection{Synthetic Data Generation}

First, the $m=n=200$ rows and columns of the matrix were each evenly partitioned into 4 random groups of 50 rows  or columns. The true side information matrices $X^*$ and $Y^*$ were set to be the indicators of the row and column groups. In recommender systems applications, we may interpret the groups as genres (for movies) or customer types. Then, we generated a random rank $\midmatstar\in\R^{4\times 4}$ with i.i.d. $N(0,1)$ entries. The matrix was then normalized to ensure $\|\midmatstar\|_{\Fr}=4$, which corresponds to entries of size $O(1)$. The sampling distribution over entries $P$ was selected randomly as follows: each entry $P^0_{u,v}$ of a matrix $P^0\in\R^{10\times 10}$ was  drawn form a uniform $U[0,1]$ distribution, and the initial matrix $P$ was set to $P=X^* P^0 {Y^*}^\top$. Then, the the full matrix $P\in\R^{m\times m}= \R^{200\times 200}$ was normalized by dividing by $|P|_1$ to ensure a valid probability distribution. The ground truth matrix was set to $X^{*}\midmatstar  {Y^*}\top$ and observed with some noise of standard deviation $0.05$. 
For each configuration of $M,N$, we performed 30 runs of the full pipeline of (1) generation of the sampling matrix $P$, (2) generation of the the $M$ unlabeled datapoints and the $N$ labeled data points, as well as (3) the training of the model (with the coefficient of nuclear norm regularization set to zero) and evaluation on a held out test set. We calculated the generalization error for each configuration for $M\in\{ 10000, 20000, \ldots, 100000\}$ and $N\in\{50, 100, 150, \ldots, 1000\}$. The values were selected to ensure that nearly perfect performance for the largest values of $M$ and $N$.

\subsection{Real Data}

We consider the following datasets: 
\begin{itemize}
    \item \textbf{Douban}~\cite{zhugraphical,zhu2019dtcdr} is a recommender systems dataset extracted from the Douban movie recommendation platform. It contains both ratings on a 1 to 5 scale and reviews (in Chinese). It contains  $4988$ users and  $4903$ items.
    \item \textbf{MovieLens 100K}~\cite{harper2015movielens} is a movie ratings dataset with 943 users on 1682 movies and 100000 ratings on a scale from 1 to 5 where each user has rated a minimum of 20 movies.
    \item \textbf{Yelp}~\cite{zhang2015character} is a popular restaurant and business reviews dataset with review ratings also on a scale from 1 to 5. 
\end{itemize}

In all three cases, we consider a semi-supervised learning setting where we remove a fraction $p$ of the labels. We train the baseline methods (which are all explicit feedback methods) on the labeled data only. For DAMC, train a shallow autoencoder with four layers of widths $400, 40, 40$, and $400$ respectively on the implicit feedback data with the binary cross entropy loss and and rely on the last 40 dimensional representation to build the matrices $X,Y$ in a downstream classic Inductive Matrix Completion with nuclear norm regularization 

We evaluated the following classic baselines: 
\begin{itemize}
    \item \textbf{UserKNN}~\cite{ahuja2019movie} a classic collaborative filtering method which predicts the average rating given to the target item by the $k$ most similar users. 
    \textbf{SoftImpute ~\citep{softimpute,softimputeALS}:} a classic matrix completion method that uses nuclear-norm regularization.
	\item \textbf{IGMC~\citep{Zhang2020Inductive}. :} this method trains a graph neural network (GNN) after building 1-hop subraphs along the adjacency matrix. The hyperparameters were tuned according to Section~5 of the original reference.
\end{itemize}

\section{Classic Results Required for the Proofs}

In this section, we recall some well-known concentration inequalities and results on subspace recovery. The results are included for the benefit of selff-containedness and no claim of originality is made for this section.

\begin{proposition}[Weldin's sin theorem , cf.~\citet{spectralmagic}, Theorem 2.9]
	\label{prop:specralmag}
	Let $M=M^*+E$ be a matrix, where $E$ is a perturbation. Let $U\Sigma V^\top$ (resp. $U^{*}\Sigma^{*} [ V^{*}]^\top$) denote the $r$-rank truncated singular value decomposition of $M$ (resp. $M^*$). Let $\sigma^*_1,\ldots,\sigma^*_m$ denote all the singular values of $M^*$ (including zeros). If $\|E\|\leq (1-1/\sqrt{2})\left[\sigma^*_r-\sigma^*_{r+1}\right]$, then we have 
	\begin{align} 
		\max\left(  \min_{R\in\mathcal{O}^{r\times r}}\left\|   UR-U^{*} \right\| , \min_{R\in\mathcal{O}^{r\times r}} \left\|   VR-V^{*} \right\|  \right) \leq \frac{2\|E\|}{\sigma^{*}_r-\sigma^{*}_{r+1}}.
	\end{align}
	
\end{proposition}

\begin{proposition}[Non commutative Bernstein inequality, Cf.~\citet{SimplerMC}]
	\label{bernstein}
	Let $X_1,\ldots,X_S$ be independent, zero mean random matrices of dimension $m\times n$. For all $k$, assume $\|X_k\|\leq M$ almost surely, and denote $\rho	_k^2=\max(\|\E(X_kX_k^\top) \|,\|\E(X^\top_kX_k) \|)$ and $\nu^2=\sum_{k}\rho_k^2$. For any $\tau>0$,
	\begin{align}
		&\mathbb{P}\left(\left \|\sum_{k=1}^S X_k \right \|\geq \tau \right)\leq (m+n)\exp\left(  -\frac     {\tau^2/2}     {\sum_{k=1}^S\rho_k^2 +M\tau/3}      \right).
	\end{align}
\end{proposition}
The following high probability version is a direct consequence. 

\begin{proposition}[High-probability version of Bernstein inequality]
	\label{bernsteinprob}
	Let $X_1,\ldots,X_S$ be independent, zero mean random matrices of dimension $m\times n$. For all $k$, assume $\|X_k\|\leq M$ almost surely, and denote $\rho_k^2=\max(\|\E(X_kX_k^\top) \|,\|\E(X^\top_kX_k) \|)$ and $\nu^2=\sum_{k}\rho_k^2$. Writing $\sigma^2=\sum_{k=1}^S\rho_k^2$, for any $\delta >0$,  we have, with probability greater than $1-\delta$: 
	\begin{align}
		\left\|\sum_{k=1}^S X_k\right\| \leq \sqrt{8/3}\sigma \sqrt{\log\left(\frac{m+n}{\delta }\right)}+ \frac{8M}{3}\log\left(\frac{m+n}{\delta}\right).
	\end{align}
	
\end{proposition}

\begin{proposition}[Bernstein inequality for non-centered random variables]
	\label{prop:noncentered}
	Let $X_1,\ldots,X_N$ be $N$ independent random variables such that $|X_i|\leq b$ almost surely. Let $v=\frac{1}{N}\sum_{i=1}^N\E(X_i^2)$. With probability greater than $1-\delta$, we have 
	\begin{align}
		\frac{1}{N}\sum_{i=1}^N X_i-\E(X_i)\leq 2b\frac{\log\lb \frac{1}{\delta }\rb}{3N}+\sqrt{\frac{2v\log\lb \frac{1}{\delta }\rb}{N}}.
	\end{align}
	
\end{proposition}
\begin{proof}
	This follows immediately from Corollary 2.11 and the comments below it in~\citet{concentration}.
\end{proof}

\begin{lemma}[Lemma 6 in~\citet{softimpute}]
	\label{lem:softimpute}
	For any matrix $Z$, the following holds:
	\[
	\|Z\|_* = \min_{U,V : Z = UV^T} \frac{1}{2} \left( \|U\|_F^2 + \|V\|_F^2 \right).
	\]
	If $\mathrm{rank}(Z) = k \leq \min\{m,n\}$, then the minimum above is attained at a factor decomposition $Z = U_{m \times k} V_{n \times k}^T$.
\end{lemma}

\section{Hardware Specifications}
\begin{table*}[!hbt]
  \centering
  \begin{tabular}{ccc}
    \toprule
    \textbf{Hardware} & \textbf{Specification} \\
    \midrule 
    OS & Rocky Linux 9.5 (Blue Onyx) \\
    CPU & AMD EPYC 9354 32-Core Processor  \\
    GPU & NVIDIA RTX A5000  \\
    RAM & 32GB  \\
    GPU memory & 24GB  \\
    \bottomrule 
  \end{tabular}
  \caption{Summary of hardware specifications used for the experiments.}
  \label{tab:hardware_software_specs}
\end{table*}


\end{document}